\documentclass[letterpaper]{article} 
\usepackage{aaai24}  
\usepackage{times}  
\usepackage{helvet}  
\usepackage{courier}  
\usepackage[hyphens]{url}  
\usepackage{graphicx} 
\usepackage{natbib}  
\usepackage{caption} 
\DeclareCaptionStyle{ruled}{labelfont=normalfont,labelsep=colon,strut=off} 
\usepackage{xspace}
\usepackage{tabularx}
\usepackage{subcaption}
\usepackage{booktabs}
\usepackage{amsmath}
\usepackage{soul}
\usepackage{algorithm}
\usepackage[noend]{algorithmic}
\usepackage{xcolor}
\usepackage{amsmath}
\usepackage{amsfonts}
\usepackage{subcaption}
\usepackage[textsize=tiny]{todonotes}
\usepackage{makecell}
\usepackage{color, colortbl}
\usepackage{multirow}

\definecolor{grayline}{RGB}{244,246,246}

\usepackage{amsthm}
\usepackage{thmtools}
\usepackage{thm-restate}
\newtheorem{definition}{Definition}

\newtheorem{theorem}{Theorem}

\usepackage{macros}

\title{Generalising Planning Environment Redesign}

\author{
	Alberto Pozanco\textsuperscript{\rm 1}\equalcontrib,
    Ramon Fraga Pereira\textsuperscript{\rm 2}\equalcontrib, and
    Daniel Borrajo\textsuperscript{\rm 1}\footnote{On leave from Universidad Carlos III de Madrid.}
}
\affiliations{
	\textsuperscript{\rm 1}J.P. Morgan AI Research\\
    \textsuperscript{\rm 2}University of Manchester, UK\\
    \{alberto.pozancolancho,daniel.borrajo\}@jpmorgan.com \\ ramon.fragapereira@manchester.ac.uk
}

\begin{document}

\maketitle

\begin{abstract}

In \textit{Environment Design}, one interested party seeks to affect another agent's decisions by applying changes to the environment.
Most research on planning environment (re)design assumes the interested party's objective is to facilitate the recognition of goals and plans, and search over the space of environment modifications to find the minimal set of changes that simplify those tasks and optimise a particular metric.
This search space is usually intractable, so existing approaches devise metric-dependent pruning techniques for performing search more efficiently.
This results in approaches that are not able to generalise across different objectives and/or metrics.
In this paper, we argue that the interested party could have objectives and metrics that are not necessarily related to recognising agents' goals or plans. Thus, to generalise the task of \textit{Planning Environment Redesign}, we develop a general environment redesign approach that is \textit{metric-agnostic} and leverages recent research on top-quality planning to efficiently redesign planning environments according to any interested party's objective and metric.
Experiments over a set of environment redesign benchmarks show that our general approach outperforms existing approaches when using well-known metrics, such as facilitating the recognition of goals, as well as its effectiveness when solving environment redesign tasks that optimise a novel set of different metrics.

\end{abstract}

\section*{Introduction}\label{sec:introduction}

In \textit{Environment Design}~\cite{ERD_IJCAI_ZhangCP09}, one interested party (usually referred to as observer) seeks to affect another agent's decisions by applying a minimal set of changes to the environment.
Most research on planning environment (re)design has focused on cooperative and adversarial settings where the observer aims to facilitate \textit{Goal or Plan Recognition}~\cite{RamirezGeffnerIJCAI2009}, i.e., infer the agent's goal or plan as soon as possible.
These tasks are known as \textit{Goal Recognition Design} (\grd)~\cite{ICAPS_KerenGK14} and \textit{Plan Recognition Design} (\prd)~\cite{TIST_MirskyGSK19}, respectively, and they have been studied under different interested party objectives, metrics, and observer's capabilities~\cite{IJCAI_KerenGK16,kulkarni2019unified,AGR_ShvoM20}, as well as agents' intentions and environment assumptions~\cite{keren2021goal}.

Existing research on planning environment design defines a \textit{metric} that is able to assess how long (in terms of action progress, i.e., number of observations) an agent can act in an environment without revealing its intended goal or plan to the observer~\cite{keren2021goal}.
Optimising these metrics can force the agent's behaviour to be more \textit{transparent},
\textit{ambiguous}, or endow \textit{predictability}~\cite{Chakraborti2019explicability}.
Most approaches assume that the environment can only be modified by removing actions. So, they search over the space of actions' removal to compute the best set of environment changes for a given metric~\cite{keren2021goal}.
Since this space is usually intractable, existing approaches devise different heuristics and pruning techniques to perform search efficiently depending on the specific metric to be optimised.
This results in \textit{metric-dependent} approaches that are not robust enough to generalise across different environment redesign metrics.

In this paper, we propose a \textit{metric-agnostic} approach to redesign \textit{fully observable} and \textit{deterministic} planning environments.
Our main contributions are twofold, as follows:

\begin{itemize}
    \item \textit{Environment Redesign} has mainly focused on promoting or impeding the recognition of goals/plans. We argue that the interested party could have other objectives that are not necessarily related to identifying the agent's goal/plan. For example, the interested party might want to redesign the environment such that the agent is constrained to follow plans that keep certain relationships with some states. This can be beneficial in many planning settings, such as \textit{Anticipatory Planning}~\cite{DBLP:conf/aips/BurnsBRYD12,aicomm18-anticipatory,arxiv22-counterplanning}, \textit{Counterplanning}~\cite{PozancoCounterPlanningEFB18}, \textit{Risk Avoidance and Management}~\cite{sohrabi2018ai}, or \textit{Planning for Opportunities}~\cite{DBLP:conf/aips/BorrajoV21,iros21}. Thus, we propose \textit{novel metrics} that can be used to redesign environments for these other settings.
   
    \item To generate new environments that optimize our \textit{novel metrics}, as well as existing metrics in the literature, we propose \approach, a \textit{General Environment Redesign} approach that employs an anytime \textit{Breadth-First Search} (BFS) algorithm. It exploits recent research on \textit{top-quality} planning~\cite{katz2020top} to improve efficiency. While previous approaches have also used BFS to explore the space of environment modifications, they assume the extremes of the spectrum. \citet{ICAPS_KerenGK14} do not assume a \textit{plan-library}, so they have to explore the state space and reason over the quality of the plans and the metric value in the environment induced by the current modifications. This yields very costly approaches that are not able to scale to complex environments with many goals. In contrast,~\citet{TIST_MirskyGSK19}'s approach assumes a hand-crafted plan-library is provided as input, which allows the algorithm to reduce the action's removal space by just considering the actions appearing in the plan-library. We propose a middle-ground approach, in which the action space is pruned by a plan-library that is not explicitly given as input, but computed using \textit{top-quality} planning~\cite{katz2020top}. 
\end{itemize}

We evaluate \approach in a set of benchmarks for environment redesign, and show that it outperforms existing approaches~\cite{ICAPS_KerenGK14} (being orders of magnitude faster) in known redesigning tasks such as \grd. We also show its effectiveness when solving environment redesign tasks that optimise a novel set of different metrics.

\section*{Background}\label{sec:background}

\textit{Planning} is the task of devising a sequence of actions (i.e., a \textit{plan}) to achieve a goal state from an initial state~\cite{GeffnerBonet13_PlanningBook}.
We follow the formalism and assumptions of the \textit{Classical Planning} setting, and assume that an environment is \textit{discrete}, \textit{fully observable}, and \textit{deterministic}.

A \textit{planning domain} $\domain$ is defined as $\langle \facts, \actions \rangle$, where: $\facts$ is a set of \textit{facts}; $\actions$ is a set of \textit{actions}, where every action $a \in \actions$ has a set of preconditions, \textit{add} and \textit{delete} effects, $\pre(a), \add(a), \del(a)$, and a \textit{positive cost}, denoted as $\cost(a)$.
We define a \textit{state} $\state$ as a finite set of positive facts $f \in \facts$ by following the \textit{closed world assumption}, so that if $f \in \state$, then $f$ is true in $\state$. We also assume a simple inference relation $\models$ such that $\state \models f$ iff $f \in \state$, $\state \not\models f$ iff $f \not \in \state$, and $\state \models f_0 \land ... \land f_n$ iff $\{f_0, ..., f_n\} \subseteq \state$.
An action $a \in \actions$ is applicable to a state $\state$ iff  $\state \models \pre(a)$, and it generates a new successor state $\state'$ by applying $a$ in $\state$, such that $\state' = (\state\setminus\del(a))\cup\add(a)$. 

A \textit{planning problem} $\problem$ is defined as $\langle \domain, \initialstate, \goal \rangle$, where: $\domain$ is a planning domain as we described above; $\initialstate \subseteq \facts$ is the \textit{initial state}; and $\goal \subseteq \facts$ is the \textit{goal state}.
A \textit{solution} to $\problem$ is a \textit{plan} $\plan = [a_0, a_1, ..., a_n]$ that maps $\initialstate$ into a state $\state$ that holds $\goal$, i.e., $\state \models G$. 
The cost of a plan $\plan$ is $\cost(\plan) = \Sigma~\cost(a_{i})$, and a plan $\plan^{*}$ is \textit{optimal} (with minimal cost) if there exists no other solution $\plan$ for $\problem$ such that $\cost(\plan) < \cost(\plan^{*})$. 
We use $h^*(s, G)$ to refer to the cost of an optimal plan of achieving $G$ from $s$.

We refer to $\allplans(\problem,b)$ as the \textit{set of all plans} that solve a planning problem $\problem$ \textit{within a sub-optimality bound} $b$~\cite{katz2020top}. 
This bound is defined as the cost of a plan $\plan$ with respect to the cost of an optimal plan $\optimalplan$, i.e., $b=\frac{\plan}{\optimalplan}$.
Therefore, $\allplans(\problem,1.0)$ will give us all the optimal plans that solve $\problem$, and $\allplans(\problem,1.5)$ will give us all the plans that solve $\problem$ within a sub-optimality bound of $1.5$.
When $b > 1$, plans in $\allplans(\problem, b)$ might contain loops, i.e., they visit at least one state more than once.
In the rest of the paper we assume that $\allplans(\problem, b)$ only contains loop-less plans~\cite{von2022loopless}. 

\section*{Planning Environment Redesign}\label{sec:planning_environment_redesign}

\textit{Planning Environment Redesign} is the task in which an interested party (observer) aims to perform off-line modifications to a \textit{planning environment} (or just \textit{environment}) where another agent will be acting, in order to constraint its potential behaviour.
Following the formalism of~\cite{ICAPS_KerenGK14}, we define a \textit{planning environment} with deterministic actions under fully observability, as follows:
\begin{definition}\label{def:planning_environment}
    A \textbf{planning environment} is a tuple $\env = \langle \envplanning = \langle \facts, \actions, \initialstate \rangle, \goals \rangle$ where $\facts, \actions$ and $\initialstate$ are the same as in a planning problem, and $\goals$ is a set of possible reachable goals $\{ G_0, G_1, ..., G_n \}$ that are of interest to either the observer or the agent.
\end{definition}

We define the \textit{planning environment redesign} problem in Definition~\ref{def:environment_redesign}, and its solutions~in~Definitions~\ref{def:environment_redesign_solution}~and ~\ref{def:environment_redesign_optimal_solution}.

\begin{definition}\label{def:environment_redesign}
A \textbf{planning environment redesign problem} is a tuple $\envred = \langle \env, M_b \rangle$ where $\env$ is the current planning environment, and $M_b$ is a metric to be optimised in order to get the new redesigned environment, assuming the agent's behaviour sub-optimality is bounded by a constant $b$.
\end{definition}

This definition is more general than the one in~\cite{JARI_KerenGK19,TIST_MirskyGSK19}, as we include the metric $M_b$ in the definition, making the problem definition \textit{metric-agnostic}.
We do not make any assumption on the relation between the observer and the agent, i.e., they could be competing, cooperating, or indifferent.

\begin{definition}\label{def:environment_redesign_solution}
A \textbf{solution} to an environment redesign problem ~$\envred = \langle \env, M_b \rangle$ is a new redesigned environment $\env' = \langle \envplanning' = \langle \facts, \actions', \initialstate \rangle, \goals \rangle$ where $\env'$ contains a new set of actions, $\actions'$, and all goals in $\goals$ are still reachable using $\envplanning'$.
\end{definition}

Although environments could be redesigned by adding or removing any element in $\env$, we follow~\cite{JARI_KerenGK19} and~\cite{TIST_MirskyGSK19}, and assume that environments are redesigned through \textit{action removal}. Thus, $\actions' = \actions \setminus \actions_\neg$, where $\actions_\neg$ is the removed actions from $\actions$.

\begin{definition}\label{def:environment_redesign_optimal_solution}
An \textbf{optimal solution} to a planning environment redesign problem~$\envred = \langle \env, M_b \rangle$ is a redesigned planning environment $\env^{*}= \langle \envplanning' = \langle \facts, \actions', \initialstate \rangle, \goals \rangle$ that optimises the given redesign metric $M_b$ while minimising $|\actions_\neg|$.
\end{definition}

An optimal solution to an environment redesign problem $\envred$ optimises the given metric $M_b$, breaking ties in favour of solutions requiring less changes to the environment.

\section*{Environment Redesign Metrics}\label{subsec:metrics}

Before proceeding to define the redesign metrics, we first provide some common notation and introduce the running example we use throughout the paper.

The \textit{metrics} we use for environment redesign rely on reasoning about sets of plans for the possible goals $\goals$ in $\envred$, and we refer to these sets of plans as a \textit{plan-library} $\planlibrary$ (following the terminology of \textit{set of plans} defined in Section~\ref{sec:background}). 
We formally define a \textit{plan-library} $\planlibrary$ in Definition~\ref{def:plan_library}.

\begin{definition}\label{def:plan_library}
Given an environment $\env$ and a sub-optimality bound $b$, a \textbf{plan-library} of a planning environment with a bound $b$ is defined as $\planlibrary(\env, b) = \bigcup_{G_i \in \goals} \Pi(\langle \envplanning, \{G_i\} \rangle, b)$.
\end{definition}

Redesign metrics often relate to \textbf{plan prefixes} of a given size $n$, i.e., the first $n$ actions of a plan $\pi$.
We use $\planprefix_n$ to refer to the first $n$ actions of a plan $\pi$. 
Similarly, we use $\Pi_n$ to denote all the plan prefixes of size $n$ of a given set of plans $\Pi$.
We abuse the notation and say that a plan prefix is inside a set of plans ($\planprefix_n \in \Pi$) iff there exists a plan $\pi \in \Pi$ for which $\planprefix_n$ is a plan prefix.
We assume the actions $\actions$ have a uniform cost equal to 1, but the metrics we define here are not limited to uniform cost.
In order to simplify notation, we use $x', x''$ when referring to two different elements in a set, i.e., $x' \neq x''$.
We also use $x \in (X', X'')$ to denote that $x \in  X' \land x \in  X''$.

\begin{figure}[t!]
    \centering
    \includegraphics[scale=0.2]{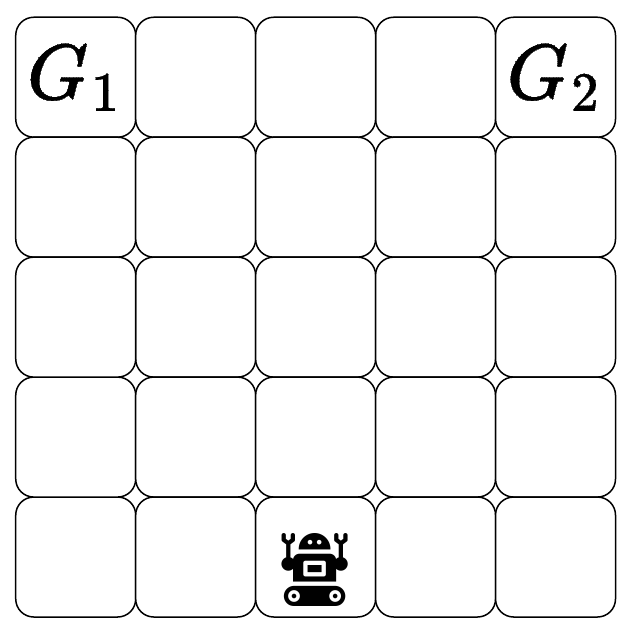}
    \caption{\textsc{grid} environment where the agent located at cell $(2,0)$ and has two possible goals: $G_1 = (0,4)$; $G_2 = (4,4)$.}
    \label{fig:running_example}
\end{figure}
As a \textit{running example}, we use the \textsc{grid} environment shown in Figure~\ref{fig:running_example}, where a robot can move in the four cardinal directions, and its possible goals consist of reaching the cells depicted by $G_1$ and $G_2$.
We use $(x,y)$ coordinates when referring to cells in the grid.
When formalising the redesign metrics, we assume optimal agents' behaviour,  so agents only follow optimal plans to achieve their goals ($b=1.0$ when computing sets of plans).

\subsection{Redesign Metrics}

We now formally define a set of environment redesign metrics, in which two of them are well-known in the literature~\cite{ICAPS_KerenGK14,TIST_MirskyGSK19}, and the other ones are our \textit{novel redesign metrics}.

\begin{figure*}[t!]
     \centering
     \begin{subfigure}[b]{0.12\textwidth}
         \centering
         \includegraphics[width=\columnwidth]{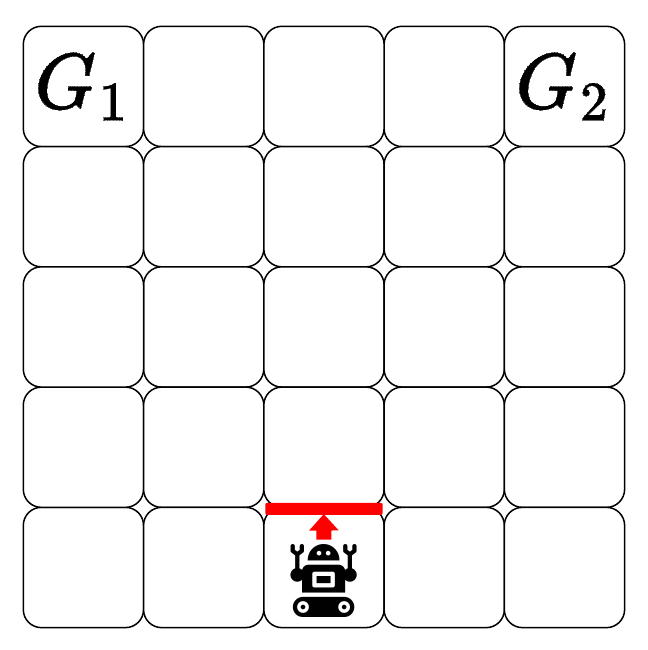}
         \caption{\it Goal \\ Transparency.}
         \label{fig:goal_transparency}
     \end{subfigure}
     \hfill
     \begin{subfigure}[b]{0.12\textwidth}
         \centering
         \includegraphics[width=\columnwidth]{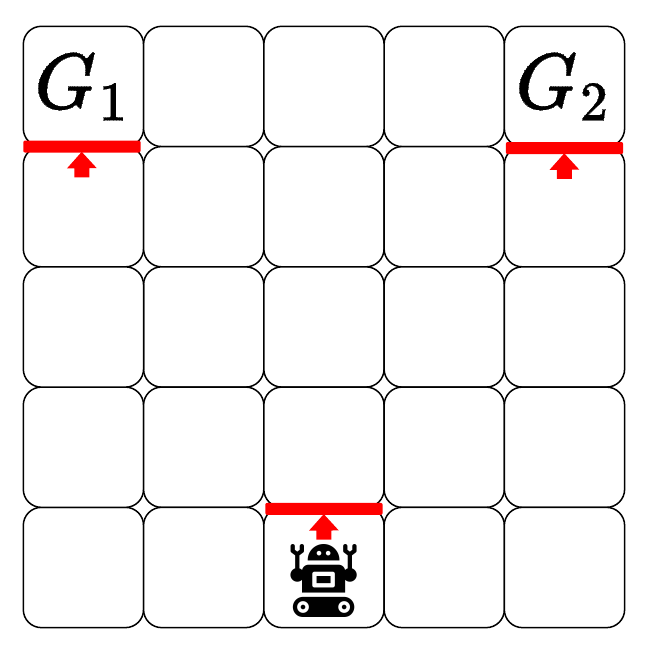}
         \caption{\it Plan \\ Transparency.}
         \label{fig:plan_transparency}
     \end{subfigure}
     \hfill
     \begin{subfigure}[b]{0.12\textwidth}
         \centering
         \includegraphics[width=\columnwidth]{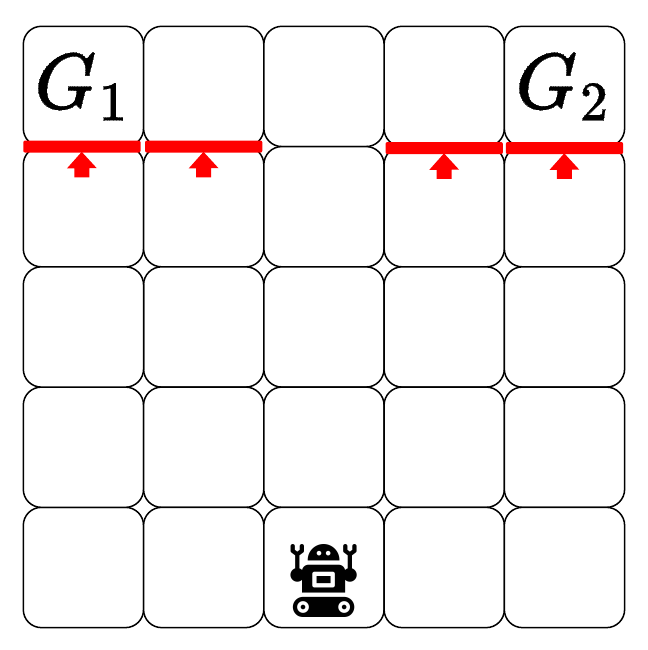}
         \caption{\it Goal \\ Privacy.}
         \label{fig:goal_privacy}
     \end{subfigure}
     \hfill
     \begin{subfigure}[b]{0.12\textwidth}
         \centering
         \includegraphics[width=\columnwidth]{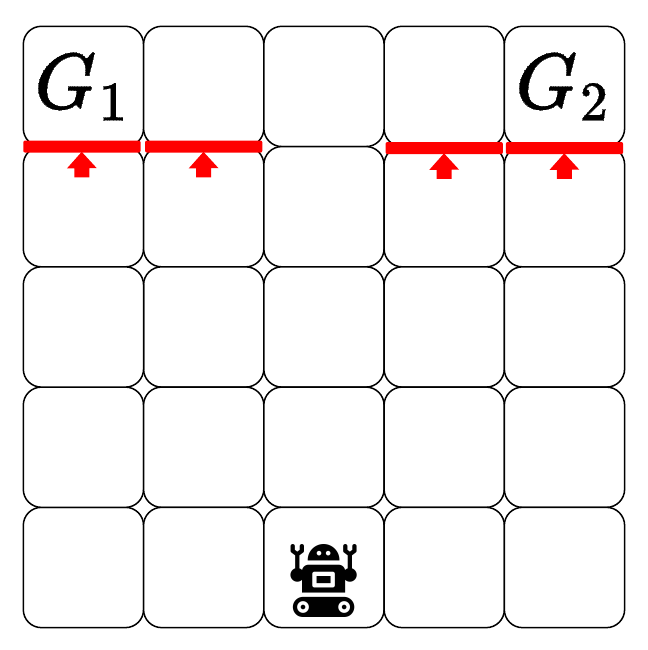}
         \caption{\it Plan \\ Privacy.}
         \label{fig:plan_privacy}
     \end{subfigure}
     \hfill
\begin{subfigure}[b]{0.12\textwidth}
         \centering
         \includegraphics[width=\columnwidth]{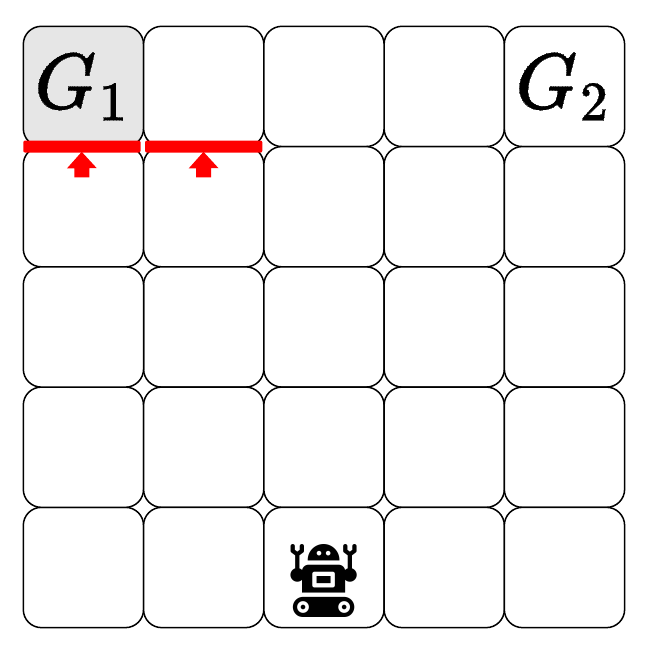}
         \caption{\it Min. Avg. Distance.}
         \label{fig:min_avg_distance}
     \end{subfigure}
     \hfill
     \begin{subfigure}[b]{0.12\textwidth}
         \centering
         \includegraphics[width=\columnwidth]{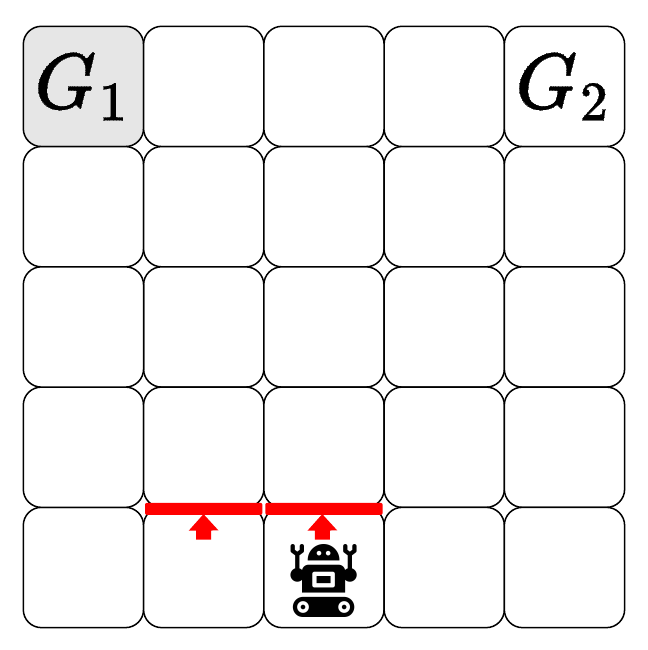}
         \caption{\it Max. Avg. Distance.}
         \label{fig:max_avg_distance}
     \end{subfigure}
     \hfill
     \begin{subfigure}[b]{0.12\textwidth}
         \centering
         \includegraphics[width=\columnwidth]{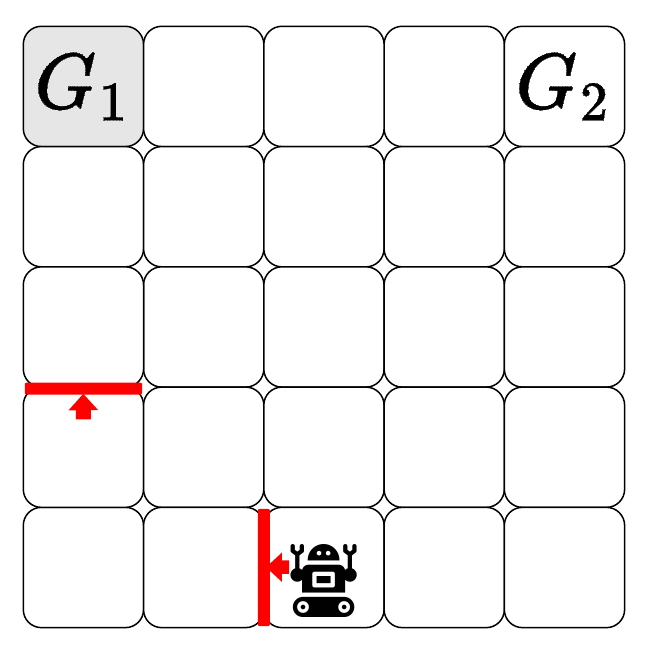}
         \caption{\it Min. Max. Distance.}
         \label{fig:min_max_distance}
     \end{subfigure}
     \hfill
     \begin{subfigure}[b]{0.12\textwidth}
         \centering
         \includegraphics[width=\columnwidth]{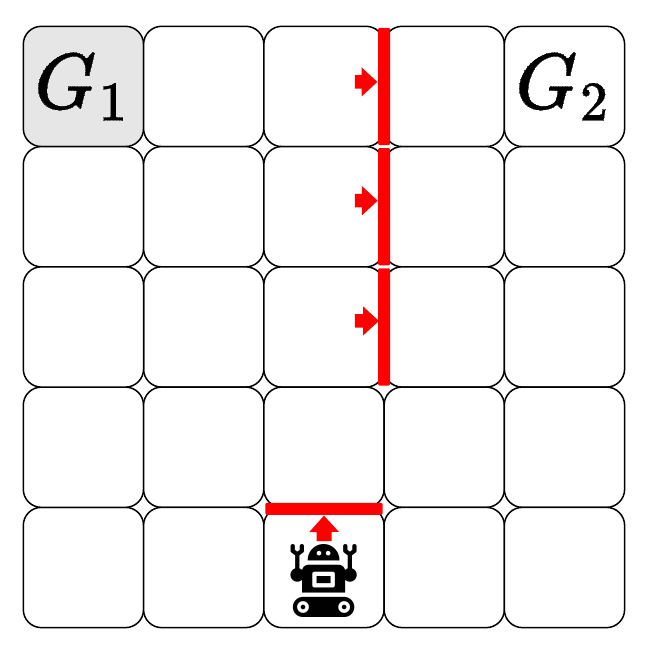}
         \caption{\it Max. Min. Distance.}
         \label{fig:max_min_distance}
     \end{subfigure}
     \caption{Redesigned environments for different metrics by using our approach with a time limit of $15$ minutes and a maximum number of removed actions $|\actions_\neg|=4$. Red arrows and lines indicate removed actions $\actions_\neg$. Intended goals are depicted in grey.}
        \label{fig:metrics}
\end{figure*}

\subsubsection{Goal Transparency (\goaltransparency).}
\textit{Goal Transparency} (equivalent to \grd) aims at redesigning an environment such that an observer can infer agents' (or humans') true intended goal as soon as possible.
This is useful in many applications such as \textit{transparent planning}~\cite{ActionSelection_MacNallyLRP18}, \textit{human-robot collaboration}~\cite{kulkarni2020designing}, or \textit{counterplanning}~\cite{PozancoCounterPlanningEFB18}.
\textit{Goal Transparency} can be achieved by minimising the \textit{worst case distinctiveness} (\textit{wcd}) of an environment $\env$~\cite{ICAPS_KerenGK14,JARI_KerenGK19}. We adapt the notation of~\citet{ICAPS_KerenGK14,JARI_KerenGK19} and formally define \textit{wcd} as follows:

\begin{definition}\label{def:wcd_goal}
    Given a planning environment $\env$ in $\envred = \langle \env = \langle \envplanning, \goals \rangle, M_b \rangle$, let $\Pi' = \Pi(\langle \envplanning, \goal' \rangle, b)$, and $\Pi'' = \Pi(\langle \envplanning, \goal'' \rangle, b)$ for $\goal', \goal'' \in \goals$. The \textbf{worst case distinctiveness (wcd)} of a pair of goals $\goal', \goal''$ is the length of the longest plan prefix $\planprefix$ that is present in $\Pi'$ and $\Pi''$:
    \begin{equation*}
\mbox{wcd}(\goal',\goal'') = \max \{ n \mid \planprefix_n \in (\Pi' \cap \Pi'')\}
    \end{equation*}
    Thus, the \textbf{worst case distinctiveness of a planning environment} is denoted as wcd($\env$), and defined as:
    \begin{equation*}
        \mbox{wcd}(\env) = \max_{\goal', \goal'' \in \goals} \mbox{wcd}(\goal', \goal'')
    \end{equation*}
\end{definition}

The \textit{wcd} of the original environment shown in Figure~\ref{fig:running_example} is $4$, and the agent can execute $4$ different actions (moving up $4$ times) without revealing its actual goal.
Figure~\ref{fig:goal_transparency} shows an optimal solution of the environment redesign problem where \textit{Goal Transparency} is optimised ($M_{1.0}$), $\textit{wcd}=0$, and $|\actions_\neg|=1$.
In this new environment, an observer will be able to recognise the agent's goal as soon as it executes the first action because the removal of the action to move from $(2,0)$ to $(2,1)$ forces the agent to move left or right, thus revealing its goal.
\textit{Goal Transparency} can also accommodate sub-optimal agents by adjusting the bound $b$, thus being useful for related tasks such as avoiding/preventing \textit{goal obfuscation}~\cite{GoalObsf_BernardiniFF20} and \textit{deception}~\cite{MastersDeceptionS17,PriceEtAl_AAMAS23}.

\subsubsection{Plan Transparency (\plantransparency).} 
One could aim to redesign an environment such that an observer can infer the agents (or humans) intended plans as soon as possible. We define such a task as \textit{Plan Transparency} (equivalent to \prd).
This is a stricter variant of \textit{Goal Transparency}, so its applications are the same, and it can be achieved by minimising the \textit{worst case plan distinctiveness} (\textit{wcpd}) of an environment $\env$~\cite{TIST_MirskyGSK19}.
We adapt the notation in~\cite{TIST_MirskyGSK19} and formally define \textit{wcpd} as follows:

\begin{definition}\label{def:wcpd_plan}
    Given $\envred = \langle \env = \langle \envplanning, \goals \rangle, M_b \rangle$, let $\pi',\pi'' \in \planlibrary(\env,b)$. The \textbf{worst case plan distinctiveness (wcpd)} of $\pi', \pi''$ is the length of the longest plan prefix $\planprefix$ in $\pi'$ and $\pi''$:
    \begin{equation*}
\mbox{wcpd}(\pi', \pi'') = \max \{ n \mid \planprefix_n \in (\{\pi'\},\{\pi''\})\}
    \end{equation*}
    Thus, the \textbf{worst case plan distinctiveness of a planning environment}, denoted as wcpd($\env$), is defined as:
    \begin{equation*}
        \mbox{wcpd}(\env) = \max_{\pi', \pi'' \in \planlibrary(\env, b)} \mbox{wcpd}(\pi',\pi'')
    \end{equation*}
\end{definition}

The $wcpd$ of the original environment shown in Figure~\ref{fig:running_example} is $4$: the agent can execute $4$ actions (moving up $4$ times) without revealing its intended plan. 
Figure~\ref{fig:plan_transparency} shows an optimal solution for this problem, where \textit{Plan Transparency} is optimised ($M_{1.0} = \mbox{PT}$), $wcpd=0$, and $|\actions_\neg| = 3$.
In this new environment, an observer will be able to recognise the agent's plan as soon as it executes the first action, as now the agent has only one optimal plan available to achieve the goals.
\citeauthor{TIST_MirskyGSK19}~\shortcite{TIST_MirskyGSK19} prove two important properties for \textit{wcpd}:
(1) the \textit{wcd} of two goals is equal to the maximum \textit{wcpd} of the plans for achieving those goals; and (2) the \textit{wcpd} of the entire plan-library $\planlibrary(\env, 1)$ is at least as high as the \textit{wcd}($\env$).

\subsubsection{Goal Privacy (\goalprivacy).}
Sometimes, autonomous agents or humans plan and act in an environment in order to keep their goals \textit{private}. To endow \textit{Goal Privacy},
one could redesign an environment to allow agents (or humans) to keep their goals as private as possible during the execution of their plans. 
\textit{Goal Privacy} can prevent goal recognition and be useful in adversarial settings~\cite{kulkarni2019unified} such as \textit{goal obfuscation}~\cite{GoalObsf_BernardiniFF20}.
We introduce a novel metric called \textit{worst case non-distinctiveness} (\textit{wcnd)}. Then, \textit{Goal Privacy} optimization will be equivalent to maximising \textit{wcnd}.
We define \textit{wcnd} as follows:
\begin{definition}\label{def:wcnd_goal}
    Given $\envred = \langle \env = \langle \envplanning, \goals \rangle, M_b \rangle$, let $\Pi' = \Pi(\langle \envplanning, \goal' \rangle, b)$, and $\Pi'' = \Pi(\langle \envplanning, \goal'' \rangle, b)$ for $\goal', \goal'' \in \goals$. The \textbf{worst case non-distinctiveness (wcnd)} of a pair of goals $\goal', \goal''$ is the length of the shortest plan prefix $\planprefix$ for which the symmetric difference\footnote{The symmetric difference $\bigtriangleup$ of two sets of plans is an operation that returns a set that includes plans present in either of the two sets of plans but absent in their intersection.} of the plans sets $\Pi'_n$ and $\Pi''_n$ of size $n$ is empty:
    \begin{equation*}
\mbox{wcnd}(\goal',\goal'') = \min \{ n \mid  (\Pi'_n \bigtriangleup \Pi''_n) \neq \emptyset\}
    \end{equation*}
Thus, the \textbf{worst case non-distinctiveness of a planning environment} is denoted as wcnd($\env$), and defined as:
    \begin{equation*}
        \mbox{wcnd}(\env) = \min_{\goal', \goal'' \in \goals} \mbox{wcnd}(\goal', \goal'')
    \end{equation*}
\end{definition}

The \textit{wcnd} of the environment shown in Figure~\ref{fig:running_example} is $0$, and the agent can execute actions that might reveal its intended goal (moving left or right).
Figure~\ref{fig:goal_privacy} shows an optimal solution of the environment redesign problem, where \textit{Goal Privacy} is optimised ($M_{1.0}=\mbox{GP}$), $wcnd=4$, and $|\actions_\neg| = 4$.
As a result, the agent is forced to execute four actions without revealing its true intended goal.
This metric and the resulting redesigned environment, are different (and more strict) than just maximising $wcd$.
The original environment already has a maximum $wcd$ of $4$, so an optimal solution to maximise $wcd$ would be an empty solution; i.e., do not apply any modification to the environment. However, this solution would allow the agent to execute actions that could reveal its intended goal earlier, so it would not be a solution to \textit{Goal Privacy}.

\subsubsection{Plan Privacy (\planprivacy).}

When facing specific situations that entail continuous monitoring, we may want to preserve our privacy by concealing what we are doing or aim to do. To do so, one could redesign an environment such that agents can to keep their executed plans as private as possible. We define this task as \textit{Plan Privacy}.
This is a variant of \textit{Goal Privacy}, so its applications are essentially the same, and 
the redesign metrics for \textit{Goal and Plan Privacy} may have similar values depending on the problem.
To achieve \textit{Plan Privacy}, we define a novel metric called \textit{worst case plan non-distinctiveness}, denoted as \textit{wcpnd}, and it can be optimised by maximising \textit{wcpnd}.
We formally define \textit{wcpnd} as follows:

\begin{definition}\label{def:wcpnd_plan}
    Given $\envred = \langle \env = \langle \envplanning, \goals \rangle, M_b \rangle$, let $\pi',\pi'' \in \planlibrary(\env,b)$. The \textbf{worst case plan non-distinctiveness (wcpnd)} of $\pi', \pi''$
    is the length of the shortest plan prefix $\planprefix$ for which $\pi'\neq \pi''$:
    {
    \begin{equation*}
\mbox{wcpnd}(\pi',\pi'') = \min \{ n \mid \pi'_n \neq \pi''_n\}
    \end{equation*}
    }
    Thus, the \textbf{worst case plan non-distinctiveness of a planning environment model}, denoted as wcpnd($\env$), is defined as:
	{
	\begin{equation*}
    	\mbox{wcpnd}(\env) = \min_{\pi',\pi'' \in \planlibrary(\env,b)} \mbox{wcpnd}(\pi',\pi'') 
	\end{equation*}
	}
\end{definition}

The \textit{wcpnd} of the original environment shown in Figure~\ref{fig:running_example} is $0$, in which the agent can act freely without being private about its executed plans.
Figure~\ref{fig:plan_privacy} shows an optimal solution of the environment redesign problem where \textit{Plan Privacy} is optimised ($M_{1.0}=\mbox{PP}$), $wcpnd=4$, and $|\actions_\neg| = 4$.
In this resulting environment, the agent can act by executing at least four actions in an optimal plan without revealing its intended plan.
\textit{Plan Privacy} can also accommodate sub-optimal agents' behaviour by adjusting the bound $b$, thus being useful for related planning applications where sub-optimal plans play an important role, such as \textit{deceptive planning}~\cite{MastersDeceptionS17,PriceEtAl_AAMAS23}.

\subsubsection{Minimise Average Distance (\minavgd).}
Certain situations require that agents (or humans) act in an environment to stay as close as possible to certain states. To accomplish this, one could redesign an environment such that an agent would be forced to ``stay'' \textit{as close as possible} to a set of partial states whilst acting for achieving its true goal. We define this task as \textit{Minimise Average Distance}, and its applications may include \textit{anticipatory planning}~\cite{DBLP:conf/aips/BurnsBRYD12} or \textit{planning for opportunities}~\cite{DBLP:conf/aips/BorrajoV21}.
More concretely, following the example in~\cite{ICAPS_KerenGK14}, the airport operator might be interested in forcing passengers to pass through some shops on the way to their gates.
It can also be useful in surveillance tasks, where one might want to constrain the surveillance agent's behaviour to pass through places where potential monitoring tasks might dynamically arrive. 
To practically endow this, we adapt the definition of planning \textit{centroids}~\cite{Centroids_MCS_PozancoEFB19,Centroids_MCS_Karpas22} to work over plans, rather than just single states.
We define the average distance of an environment
for a set of partial states and a goal state as
\textit{avgD}, as follows:

\begin{definition}\label{def:minavgd}
    Given $\envred = \langle \env = \langle \envplanning, \goals \rangle, M_b \rangle$, where $G_t \in \goals$ is a true goal, and $\goals_S = \goals \setminus \{G_t\}$ is a set of partial states of interest to reason about. Let $\state_{\Pi}$ be all the states traversed by the plans in $\Pi=\langle \envplanning, G_t \rangle$. 
The \textbf{average distance of a planning environment} is denoted as \textit{avgD}, and defined as:
    
    \begin{equation*}
        \mbox{avgD($\env$)} = \frac{\sum_{s_i \in \state_{\Pi}, G_i \in \goals_S} h^*(s_i, G_i)}{|\state_{\Pi}| \times |\goals_S|}
    \end{equation*}
\end{definition}

The \textit{avgD} of the original environment shown in Figure~\ref{fig:running_example} is $5$.
Figure~\ref{fig:min_avg_distance} shows a solution of the original environment redesign problem where the average distance is minimised ($M_{1.0} = \textit{minAvgD}$).
In the resulting environment, two actions are removed ($|\actions_\neg| = 2$), and the agent is forced to ``stay'' as close as possible to $G_2$ whilst following an optimal plan to achieve its intended goal $G_1$.
This optimal plan involves moving north four steps, followed by $2$ east steps, and traverses $7$ states, yielding $\textit{avgD}=\frac{6+5+4+3+2+3+4}{7} =3.86$.
Even if the example only shows one special goal to reason about, $G_2$, the metric works for any set of goals.

\subsubsection{Maximise Average Distance (\maxavgd).}
One could aim to redesign an environment such that agents (alternatively, humans) would be forced to ``stay'' \textit{as far as possible} from a set of potential \textit{risks} whilst achieving their goals~\cite{perny2007state,IntexPozanco20}.
This can be useful in evacuation domains, such as in the event of a volcano eruption, where the goal is to move people to a safe place while staying as far as possible from a set of dangerous areas; or in financial planning, where the aim is to achieve the user's financial goal while staying far from financial risks such as high debt.
In these cases, it is usually impossible to completely eliminate the risk (block the goal), so our assumption about all the goals being reachable (Definition~\ref{def:environment_redesign_solution}) still holds in practice.
To do so, we can \textit{Maximise Average Distance} (\textit{maxAvgD}) using Definition~\ref{def:minavgd}. 

Figure~\ref{fig:max_avg_distance} shows a solution for the environment redesign problem in Figure~\ref{fig:running_example} when using \textit{maxAvgD}, where average distance is maximised $M_{1.0} = \textit{MaxAvgD}$, $\textit{maxAvgD}=\frac{6+7+8+7+6+5+4}{7}=6.15$, and $|\actions_\neg| = 2$.
In this case, the agent is forced to stay as far as possible from $G_2$ while following an optimal plan to achieve $G_1$.

\subsubsection{Minimise Maximum Distance (\minmaxd).}
\textit{Minimise Maximum Distance} aims at redesigning an environment such that agents are ``forced'' to never stay too far from a set of partial states whilst achieving its true intended goal. It can be used in the same previous examples.We adapt the definition of planning \textit{minimum covering states}~\cite{Centroids_MCS_PozancoEFB19} over plans, and define the maximum distance of an environment \textit{maxD} as: 

\begin{definition}\label{def:minmaxd}
    Given $\envred = \langle \env = \langle \envplanning, \goals \rangle, M_b \rangle$, where $G_t \in \goals$ is a true goal, and $\goals_S = \goals \setminus \{G_t\}$ is the set of partial states to reason about. Let $\state_{\Pi}$ be all the states traversed by the plans in $\Pi=\langle \envplanning, G_t \rangle$. 
The \textbf{maximum distance of a planning environment} is denoted as \textit{maxD}, and defined as:
    \begin{equation}
        \mbox{maxD($\env$)} = \max_{s_i \in \state_{\Pi}, G_i \in \goals_S} h^*(s_i, G_i)
    \end{equation}
\end{definition}

The \textit{maxD} of the environment shown in Figure~\ref{fig:running_example} is $8$, which is achieved when the agent visits the cell $(0,0)$.
Figure~\ref{fig:min_max_distance} shows a solution for this environment redesign problem where the maximum distance is minimised ($M_{1.0} = \textit{MinMaxD}$), then we have $\textit{maxD} = 6$ and $|\actions_\neg|=2$.
This metric is different from minimising average distance. While the solution in Figure~\ref{fig:min_avg_distance} also has a \textit{maxD} of $6$, the solution in Figure~\ref{fig:min_max_distance} does not minimise \textit{avgD}.

\subsubsection{Maximise Minimum Distance (\maxmind).}
One could redesign an environment such that agents are compelled to avoid getting too close to a set of partial states whilst achieving their true goal.
Redesigning environments to optimise this metric can be useful in the same risk avoidance and evacuation domains we already mentioned.
We define this task as \textit{Maximise Minimum Distance} (\textit{maxMinD}).
We define the minimum distance of $\env$ as \textit{minD}, as follows:

\begin{definition}\label{def:maxmind}
    Given $\envred = \langle \env = \langle \envplanning, \goals \rangle, M_b \rangle$, where $G_t \in \goals$ is a true goal, and $\goals_S = \goals \setminus \{G_t\}$ is the set of partial states to reason about. Let $\state_{\Pi}$ be all the states traversed by the plans in $\Pi=\langle \envplanning, G_t \rangle$. 
The \textbf{minimum distance of a planning environment} is denoted as \textit{minD}, and defined as:
    \begin{equation}
        \mbox{minD($\env$)} = \min_{s_i \in \state_{\Pi}, G_i \in \goals_S} h^*(s_i, G_i)
    \end{equation}
\end{definition}

The \textit{minD} of the environment shown in Figure~\ref{fig:running_example} is $2$, which is achieved when the agent visits the cell $(2,4)$.
Figure~\ref{fig:max_min_distance} shows a solution of the environment redesign problem, where the minimum distance is maximised $M_{1.0} = \textit{MaxMinD}$, $\textit{minD} = 6$, and $|\actions_\neg|=2$.

\section*{Environment Redesign via Search}\label{sec:approach}

We now present \approach, a general environment redesign approach that is \textit{metric-agnostic} and employs 
an anytime \textit{Breadth-First Search} (BFS)~\cite[Section 3.3.1]{Rssell2005ai} algorithm that exploits recent research on \textit{top-quality} planning to improve the search efficiency.

\approach~is described in Algorithm~\ref{alg:bfs_redesign}, and
takes as input an environment redesign problem $\envred = \langle \env, M_b \rangle$ and a stopping condition $C$. 
\approach searches the space of environment modifications by iteratively generating and evaluating environments where an increasing number of actions is removed.
\approach~returns the
set of best solutions found $\cal M$ until $C$ is triggered, i.e., the set of different environment modifications that optimises a redesign metric $M_b$, yielding a redesigned environment with metric value $m^+$.

\subsubsection{Compute Plan-Library $\planlibrary$ (Line~\ref{alg:bfs_redesign:compute_plan_library}).} \approach~first computes a plan-library $\planlibrary(\env,b)$ for the given environment redesign problem $\envred = \langle \env, M_b \rangle$ by calling a \textsc{topQualityPlanner}.

\renewcommand{\algorithmicrequire}{\textbf{Input:}}
\renewcommand{\algorithmicensure}{\textbf{Output:}}

\begin{algorithm}[t!]
\caption{\approach: {\footnotesize A General Environment Redesign Approach}}
\label{alg:bfs_redesign}
\small
\begin{algorithmic}[1]
\REQUIRE Redesign problem $\envred = \langle \env, M_b \rangle$, $C$ stopping condition.
\ENSURE Set of solutions found ${\cal M}$, metric value found $m^+$.

\STATE $\planlibrary(\env, b) \gets \textsc{topQualityPlanner}(\env, b)$\label{alg:bfs_redesign:compute_plan_library}

\STATE $amod \gets \textsc{getAllowedModifications}( \planlibrary(\env,b), \actions, M_b)$\label{alg:bfs_redesign:compute_allowed_modifications}

\STATE $s_0 \gets \emptyset$, $\textsc{open} \gets s_0$, ${\cal M} \gets \{s_0\}$\label{alg:bfs_redesign:search_begin}
\STATE $m_0, m^+ \gets \textsc{evaluate}(s,M, \env, b)$

\WHILE{$\neg C$}\label{alg:bfs_redesign:search_stop_condition}
    \STATE $s \gets \textsc{open.dequeue}()$ \COMMENT{\//* \textit{State s with lowest $|\actions_\neg|$} */\/}
    \FOR{$a$ \textbf{in} $amod$}
        \STATE $s' \gets s \cup a $ \label{alg:bfs_redesign:adding_a}
        \IF{$\textsc{isValid}(s')$}\label{alg:bfs_redesign:validity}
            \STATE $\textsc{open.queue}(s')$
            \STATE $m' \gets \textsc{evaluate}(s',M, \env, b)$
            \IF{$\textsc{isBetter}(m',m^+)$}\label{alg:bfs_redesign:search_metric_check}
                \STATE $m^+ \gets m'$, ${\cal M} \gets \{s'\}$
            \ELSIF{$m' = m^+$ \AND $|s'| = |s''|, ~st.~s'' \in~{\cal M}$}\label{alg:bfs_redesign:search_metric_check2}
                \STATE ${\cal M} \gets {\cal M} \cup \{s'\}$\label{alg:bfs_redesign:search_modifications}
\ENDIF
        \ENDIF
    \ENDFOR

\ENDWHILE

\RETURN ${\cal M}, m^+$\label{alg:bfs_redesign:search_end}
\end{algorithmic}
\end{algorithm}

\subsubsection{Compute Allowed Modifications (Line~\ref{alg:bfs_redesign:compute_allowed_modifications}).}
After computing $\planlibrary(\env,b)$, \approach~then computes the set of allowed modifications for the given environment by using the $\textsc{getAllowedModifications}$ function, taking as input a plan-library $\planlibrary$, a set of actions $\actions$, and a metric $M_b$ to be optimised.
Depending on the metric, this function can either return all the actions in $\actions$, or only the subset of actions that appear in the plan-library $\planlibrary$, thus pruning the space of modifications.
\approach~only reasons over the actions in the plan-library $\planlibrary$ for optimising GT (\textit{wcd}), GP (\textit{wcnd}), PT (\textit{wcpd}) or PP (\textit{wcpnd}), as removing actions that do not appear in the plan-library does not affect these metrics.
In addition, \approach~reasons over all the possible actions in $\actions$ when optimising the distance-related metrics (\textit{minAvgD}, \textit{maxAvgD}, \textit{minMaxD}, \textit{maxMinD}), as removing actions that do not appear in the agent's optimal plans that achieve the true intended goal might affect and influence directly these metrics (see Figure~\ref{fig:max_min_distance}, where $\actions_\neg$ includes actions that are not part of any optimal plan that achieves $G_1$).

\subsubsection{Searching Process (Lines \ref{alg:bfs_redesign:search_begin}--\ref{alg:bfs_redesign:search_end}).}
With the computation of the plan-library $\planlibrary(\env,b)$ and the allowed modifications properly in place,
\approach~initialises the search structures, and then conducts a BFS search until the stopping condition $C$ is met (Line~\ref{alg:bfs_redesign:search_stop_condition}).
Most existing algorithms only stop when the best possible value for a metric is achieved~\cite{JARI_KerenGK19,TIST_MirskyGSK19}.
While defining this best possible value is easy for some metrics, i.e., $wcd=0$ when optimising GT, this value is not easy to be properly defined for all metrics.
Namely, it is infeasible to know in advance the lowest or highest average distance that we can achieve when redesigning an environment.
Hence, we generalise the stopping conditions in the literature and assume $C$ can represent any formula, such as a time limit or memory limit, a bound on the number of removed actions, or an improvement ratio of the metric with respect to its original value.

In each iteration, \approach~gets the best node from the open list $\textsc{open}$ according to its $g$-$value$, defined as the size of the removed actions set $|\actions_\neg|$.
Then, \approach~generates the successors of the current node $s$ by adding removable actions in $amod$ to the current node's removed actions' set (Line~\ref{alg:bfs_redesign:adding_a}).
Before appending the new node $s'$ to $\textsc{open}$, \approach~checks if it is valid, verifying that all the goals in $\goals$ are still achievable in the resulting environment after removing the actions in $s'$.
If $s'$ is a valid node, \approach~computes the value of the metric $M_b$ for that node, $m'$, using the \textsc{evaluate} function.
This function assesses the quality of the environment obtained after removing the actions in $s'$, using for example any of the metrics proposed in Definitions~\ref{def:wcd_goal}--\ref{def:maxmind}.
If $m'$ \textsc{isBetter} than the best metric value $m^+$ found so far (lower when minimising, higher when maximising), then this value is replaced, and the set of environment modifications $\cal M$ is updated (Lines~\ref{alg:bfs_redesign:search_metric_check}--\ref{alg:bfs_redesign:search_modifications}).
If $m'$ is equal to $m^+$ and node $s'$ has the same number of modifications (same $g$-$value$) as those nodes in $\cal M$, then node $s'$ is included in the set of environment modifications $\cal M$ (Line~\ref{alg:bfs_redesign:search_modifications}).
Finally, \approach~terminates when the condition $C$ is met (Line~\ref{alg:bfs_redesign:search_stop_condition}), returning the best solutions found (environment modifications $\cal M$), and the best value $m^+$ for the redesign metric $M$ in these solutions.

\subsubsection{Theoretical Properties.} 

In Theorem~\ref{thm:ger_theory}, we give the bases and provide the guarantees to prove that our general environment redesign approach $\approach$, described in Algorithm~\ref{alg:bfs_redesign}, is \textit{sound}, \textit{complete}, and \textit{optimal} under certain assumptions.

\begin{theorem}\label{thm:ger_theory}
    Let us assume $C=\{ \textsc{open}=\emptyset \}$, $amod = \actions$, and infinite planning time and memory resources.
    Under those assumptions, \approach 
    is \textbf{sound}, \textbf{complete}, and \textbf{optimal}.
\end{theorem}

\begin{proof}
    \approach is \textit{sound} due to the fact that it checks the validity of each node (Line~\ref{alg:bfs_redesign:validity}) before adding it to \textsc{open}. Thus, it only includes in the solution set $\cal M$, nodes that entail new environments where all goals in $\goals$ are reachable (Definition~\ref{def:environment_redesign_solution}). Given infinite resources, since it explores the full state space ($amod$ contains all possible actions), if there is a solution, it will find it, so
    \approach is \textit{complete}. It is also \textit{optimal} given that it is complete, and it ensures that $\cal M$ only contains solutions with optimal metric values that entail minimal modifications to the environment (Definition~\ref{def:environment_redesign_optimal_solution}) in Lines~\ref{alg:bfs_redesign:search_metric_check} and~\ref{alg:bfs_redesign:search_metric_check2}.
\end{proof}

\approach can also keep the above theoretical properties under less restrictive assumptions, depending on the metric.
In the case of GT, PT, GP, and PP, metrics defined over how agents achieve their goals, \approach can search in a smaller space, that is defined by $amod=\planlibrary(\env, b)$. 
The only action removals that affect the value of these metrics are those in the plan-library.
Assuming the plan-library is finite, and contains all the plans up to a sub-optimality bound $b$, then \approach preserves \textit{soundness}, \textit{completeness}, and \textit{optimality}.

\section*{Experiments and Evaluation}\label{sec:experiments_evaluation}

\begin{table*}[!ht]
    \centering
    \setlength\tabcolsep{2pt}
    \renewcommand{\arraystretch}{1.2}
    \fontsize{9}{9}\selectfont

    \begin{tabular}{llllllll}
        \toprule

        & \multicolumn{7}{c}{\textbf{\goaltransparency} $\Downarrow$ (\textit{wcd})}

        \\ 
        
        \cmidrule[\heavyrulewidth]{2-8}

        & \multicolumn{3}{c}{\approach} & & \multicolumn{3}{c}{\grd-\textit{LS}}                                                          
        \\ \cmidrule[\heavyrulewidth]{2-4} \cmidrule[\heavyrulewidth]{6-8}
       
        \# {\it domain} & \multicolumn{1}{c}{\planningtime} & \multicolumn{1}{c}{\initialmetric} & \multicolumn{1}{c}{\improvement} & \multicolumn{1}{c}{} 
        & \multicolumn{1}{c}{\planningtime} & \multicolumn{1}{c}{\initialmetric} & \multicolumn{1}{c}{\improvement} 
        
        \\ \cmidrule{2-4} \cmidrule{6-8}

        \rowcolor{grayline}\blocks         
        & 1.0/0.2 & 5.2/2.1 & 3.7/1.9& & 63.7/59.5 & 5.2/2.1 & 3.7/1.9 \\
        
        \depots         
        & 1.0/0.0 & 5.0/0.0 & 4.0/0.0 && 69.0/0.0 & 5.0/0.0 & 4.0/0.0 \\
        
        \rowcolor{grayline}\gridnavigation 
        & 8.1/8.6 & 4.2/1.5 & 1.8/1.3 && 345.0/378.8 & 4.2/1.5 & 2.2/1.0 \\
        
        \ipcgrid        
        & 1.1/0.2 & 11.1/10.0 & 8.5/10.4 && 119.9/145.5 & 11.1/10.0 & 8.5/10.4\\
        
        \rowcolor{grayline}\logistics
        & - & - & - & & - & - & - \\
        \bottomrule
    \end{tabular}

    \begin{tabular}{lllllllllllllllll}
        \toprule
        
        & \multicolumn{1}{c}{} & \multicolumn{2}{l}{\bf ~\plantransparency $\Downarrow$}                                                 
        & \multicolumn{1}{c}{} & \multicolumn{3}{c}{\bf \goalprivacy $\Uparrow$}                                                 
        & \multicolumn{1}{c}{} & \multicolumn{3}{c}{\bf \planprivacy $\Uparrow$}                                        
        \\
        
        & \multicolumn{1}{c}{} & \multicolumn{2}{l}{\it (wcpd)}                                              
        & \multicolumn{1}{c}{} & \multicolumn{3}{c}{\it (wcnd)}                                              
        & \multicolumn{1}{c}{} & \multicolumn{3}{c}{\it (wcpnd)}
        \\ 
        
        \cmidrule[\heavyrulewidth]{2-4}
        \cmidrule[\heavyrulewidth]{6-8}
        \cmidrule[\heavyrulewidth]{10-12}
        
        & \multicolumn{3}{c}{\approach}                                                   &                      
        & \multicolumn{3}{c}{\approach}                                                   &                      
        & \multicolumn{3}{c}{\approach}                                                 
        \\ \cmidrule[\heavyrulewidth]{2-4}
        \cmidrule[\heavyrulewidth]{6-8}
        \cmidrule[\heavyrulewidth]{10-12}
       
        \# {\it domain} & \multicolumn{1}{c}{\planningtime} & \multicolumn{1}{c}{\initialmetric} & \multicolumn{1}{c}{\improvement} & \multicolumn{1}{c}{} 
        & \multicolumn{1}{c}{\planningtime} & \multicolumn{1}{c}{\initialmetric} & \multicolumn{1}{c}{\improvement} & \multicolumn{1}{c}{} 
        & \multicolumn{1}{c}{\planningtime} & \multicolumn{1}{c}{\initialmetric} & \multicolumn{1}{c}{\improvement} & \multicolumn{1}{c}{} 
        
        \\ \cmidrule{2-4} \cmidrule{6-8} \cmidrule{10-12}

        \rowcolor{grayline}\blocks         
        & 1.0/0.2 & 5.3/2.0 & 3.8/2.0 && 1.0/0.2 & 2.4/1.3 & 4.4/1.8 && 1.0/0.2 & 2.4/1.3 & 4.4/1.8 \\
        
        \depots         
        & 1.2/0.3 & 6.5/1.0 & 5.2/1.0 && 1.1/0.3 & 0.0/0.0 & 4.0/1.5& & - & - & - \\
        
        \rowcolor{grayline} \gridnavigation 
        & 60.1/170.7 & 4.0/1.2 & 1.8/1.3 && 19.1/74.9 & 0.0/0.0 & 1.4/0.6 && 1.4/0.4 & 0.0/0.0 & 1.4/0.6\\
        
        \ipcgrid        
        & 1.1/0.4 & 11.1/11.1 & 7.7/10.1 && 0.9/0.1 & 2.0/0.0 & 3.0/0.0 && 0.9/0.2 & 2.0/0.0 & 4.0/2.6\\
        
        \rowcolor{grayline}\logistics
        & - & - & - && 
        150.2/210.6 & 0.0/0.0 & 10.0/0.0 && - & - & -
        \\
        \bottomrule
    \end{tabular}

    \begin{tabular}{llllllllllllllllll}
        \toprule                                       
        & \multicolumn{1}{c}{} & \multicolumn{2}{l}{\bf \minavgd $\Downarrow$}                                            
        & \multicolumn{1}{c}{} & \multicolumn{3}{c}{\bf \maxavgd $\Uparrow$}                                            
        & \multicolumn{1}{c}{} & \multicolumn{3}{c}{\bf \minmaxd $\Downarrow$}                                            
        & \multicolumn{1}{c}{} & \multicolumn{3}{c}{\bf \maxmind $\Uparrow$}                                  
        \\
                                                 
        & \multicolumn{1}{c}{} & \multicolumn{2}{l}{~\it (minAvgD)}                                              
        & \multicolumn{1}{c}{} & \multicolumn{3}{c}{\it (maxAvgD)}                                              
        & \multicolumn{1}{c}{} & \multicolumn{3}{c}{\it (minMaxD)}                                              
        & \multicolumn{1}{c}{} & \multicolumn{3}{c}{\it (maxMinD)}     \\ 

        \cmidrule[\heavyrulewidth]{2-4}
        \cmidrule[\heavyrulewidth]{6-8}
        \cmidrule[\heavyrulewidth]{10-12}
        \cmidrule[\heavyrulewidth]{14-16}
                     
        & \multicolumn{3}{c}{\approach}                                                   &                      
        & \multicolumn{3}{c}{\approach}                                                   &                      
        & \multicolumn{3}{c}{\approach}                                                   &                      
        & \multicolumn{3}{c}{\approach}                                         
        
        \\ \cmidrule[\heavyrulewidth]{2-4}
        \cmidrule[\heavyrulewidth]{6-8}
        \cmidrule[\heavyrulewidth]{10-12}
        \cmidrule[\heavyrulewidth]{14-16}
       
        \# {\it domain} & \multicolumn{1}{c}{\planningtime} & \multicolumn{1}{c}{\initialmetric} & \multicolumn{1}{c}{\improvement} & \multicolumn{1}{c}{} 
        & \multicolumn{1}{c}{\planningtime} & \multicolumn{1}{c}{\initialmetric} & \multicolumn{1}{c}{\improvement} & \multicolumn{1}{c}{} 
        & \multicolumn{1}{c}{\planningtime} & \multicolumn{1}{c}{\initialmetric} & \multicolumn{1}{c}{\improvement} & \multicolumn{1}{c}{} 
        & \multicolumn{1}{c}{\planningtime} & \multicolumn{1}{c}{\initialmetric} & \multicolumn{1}{c}{\improvement} 
        
        \\ \cmidrule{2-4}
        \cmidrule{6-8}
        \cmidrule{10-12}
        \cmidrule{14-16}

        \rowcolor{grayline}\blocks         
        & 114.7/105.6 & 8.2/0.8 & 7.9/0.9 && 84.7/89.3 & 8.4/1.1 & 8.9/1.1&& 91.2/93.0 & 12.7/1.5 & 11.2/1.6 && 70.1/36.8 & 3.2/0.8 & 4.8/0.7\\
        
        \depots         
        & 52.8/28.7 & 7.0/0.9 & 6.6/1.0 && 104.3/111.7 & 7.1/0.8 & 7.5/0.8 && 48.3/22.7 & 12.1/1.4 & 9.8/0.4 && 54.0/23.4 & 2.6/2.1 & 4.0/1.9\\
        
        \rowcolor{grayline}\gridnavigation 
        & 266.0/245.1 & 4.3/0.9 & 4.0/1.0 && 383.8/255.8 & 4.3/1.0 & 4.9/1.1 && 152.0/192.1 & 7.9/1.3 & 6.7/1.5 && 232.3/175.8 & 0.7/1.0 & 2.0/1.1\\
        
        \ipcgrid        
        & 167.4/193.3 & 13.5/3.8 & 13.2/3.8 && 173.8/218.1 & 11.1/3.7 & 11.6/3.8 && 29.8/4.2 & 19.3/3.8 & 17.7/3.2 && 167.7/239.9 & 3.1/3.1 & 4.2/3.0\\
        
        \rowcolor{grayline}\logistics
        & 409.8/292.3 & 10.1/1.2 & 9.6/1.0 && 410.6/334.2 & 9.9/1.5 & 10.5/1.4 && - & - & - && 314.0/233.8 & 3.3/1.5 & 5.0/1.0 
        \\
        
        \bottomrule
    \end{tabular}
    \caption{Experimental results for the eight different redesign metrics. Each cell represents \textit{avg}/\textit{std} values for the redesign metrics. Cells with ``-'' mean that the redesign metric could not be improved for any problem in the domain within the time limit of 15 minutes (900 seconds). $\Downarrow$ represents reducing $m_0$, whereas $\Uparrow$ represents increasing $m_0$.}
    \label{tab:aggregated_results}
\end{table*} 

We now present the experiments carried out to evaluate our environment redesign approach \approach.
The aim of the evaluation is twofold: (1) compare \approach against state-of-the-art approaches for \grd~\cite{ICAPS_KerenGK14} when optimising \textit{wcd}; and (2) show \approach's performance when optimising the metrics we introduced in Section~\ref{subsec:metrics}. 

\subsubsection{Benchmarks and Setup.}

We wanted experiment and evaluate \approach in the same benchmarks introduced by~\citet{ICAPS_KerenGK14,JARI_KerenGK19}, but we were unable to obtain such benchmarks either from the public repository on GitHub nor from the authors.
Therefore, we have created a novel benchmark that contains $300$ planning environment problems equally split across the following five well-known domains: \blocks~words, \depots, \gridnavigation, \ipcgrid, and \logistics.
The number of possible goals varies in size, so we have environments with 3, 4, and 5 possible goals, and on average we have 4 possible goals over the five different benchmarks.
For the metrics where this is relevant, the true goal $G_t$ is selected as the first goal in $\goals$.
The environments are encoded in PDDL (Planning Domain Definition 
Language)~\cite{PDDLMcdermott1998}.
We generate $8$ redesign problems for each environment by varying the metric $M_b$ that should be optimised, using the metrics defined in Definitions~\ref{def:wcd_goal} to~\ref{def:maxmind}.
This gives us $300 \times 8 = 2400$ planning environment redesign problems.

\approach~uses \textsc{sym-k}~\cite{von2022loopless}, a state-of-the-art top-quality planner, to compute the plan-library.
We run \textsc{sym-k} with a bound of $1.0$, i.e., we only perform experiments with optimal agent behavior, although all our metrics support arbitrary sub-optimality bounds.
We also set a limit of $1,000$ plans to prevent disk overflows and avoid \approach spending all the time computing the plan-library in redesign problems with a large number of optimal plans.
For the subset of $300$ environment redesign problems, where the aim is minimising \textit{wcd}, we compare \approach against the most efficient \grd~approach (\textit{latest-split}) of \citet{ICAPS_KerenGK14}, denoted as \grd-\textit{LS}.
We execute this code as taken from the repository with \textsc{fast-downward}~\cite{FastDownward_Helmert06} as the planner used to solve the compiled planning problems, and with a design budget (maximum number of actions that can be jointly removed from the environment) of $5$.
Benchmarks, \approach's code, and further results are available on GitHub\footnote{\scriptsize \url{https://github.com/ramonpereira/general-environment-redesign}}.
We have run all experiments using 4vCPU AMD EPYC 7R13 Processor 2.95GHz with 32GB of RAM, and run \approach on each environment redesign problem with $C=\{$\textit{time limit} = 900s or \mbox{\textit{memory limit} = 4GB}$\}$. We used the same stopping condition $C$ for both \approach and \grd-\textit{LS}.

\subsubsection{Execution Time Results and Comparison against GRD.}

Table~\ref{tab:aggregated_results} shows our results, using the following metrics for evaluation: $T$, the time (seconds) to find the best solution; $m_0$, the metric value of the original environment; and $m^+$, the metric value of the environment returned as a solution.
We only report results for the subset of problems for which the given metric could be improved within the time and memory limits.
In the case of \goaltransparency, we only report results over commonly solved problems, i.e., those problems for which both \approach and \grd-\textit{LS} can improve the given metric.

\begin{figure*}[!ht]
    \centering
    \begin{subfigure}[b]{0.22\textwidth}
        \centering
        \includegraphics[width=\columnwidth]{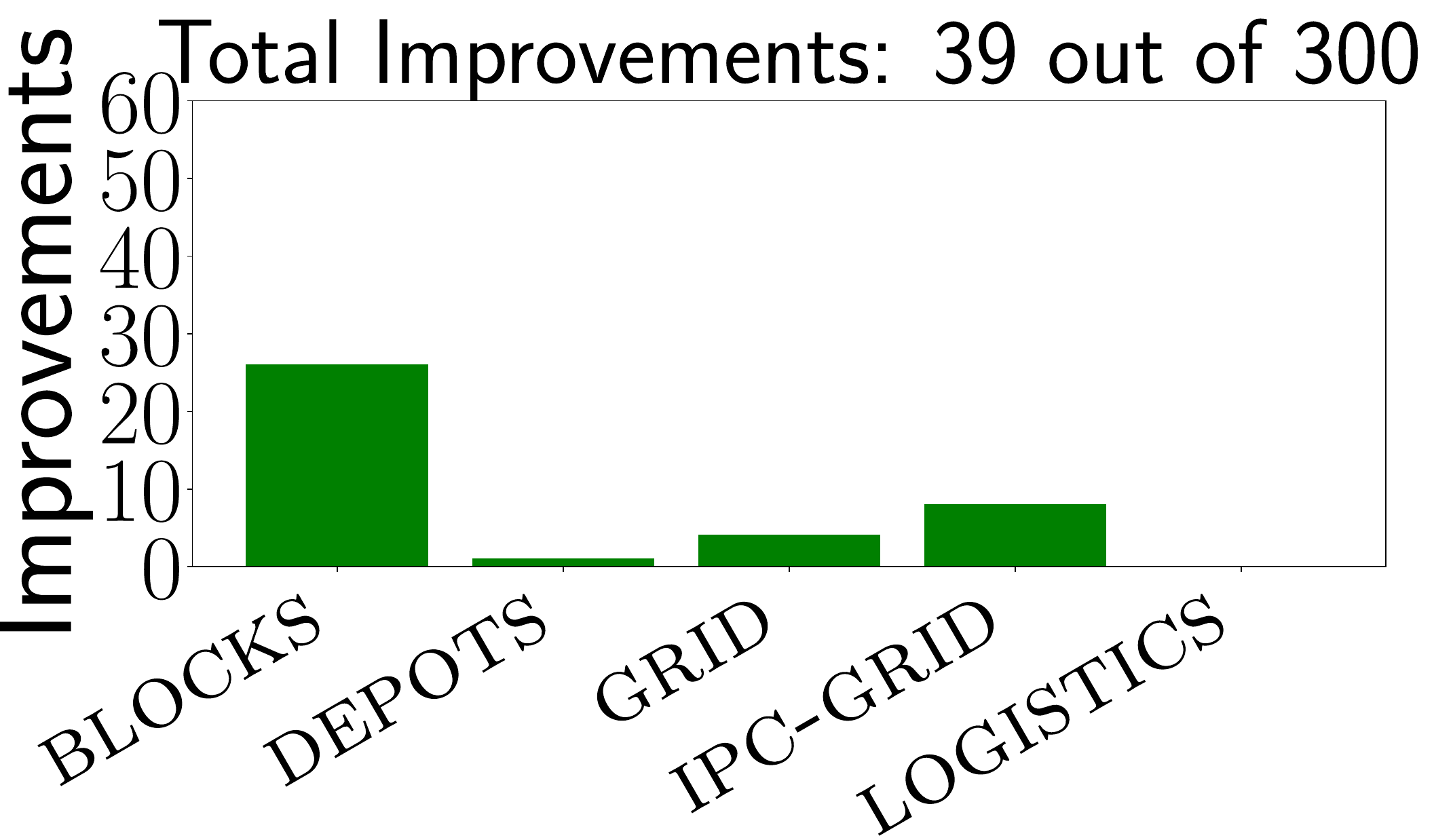}
        \caption{\it Goal Transparency.}
        \label{fig:result_goal_transparency}
    \end{subfigure}
    \hfill
    \begin{subfigure}[b]{0.22\textwidth}
        \centering
        \includegraphics[width=\columnwidth]{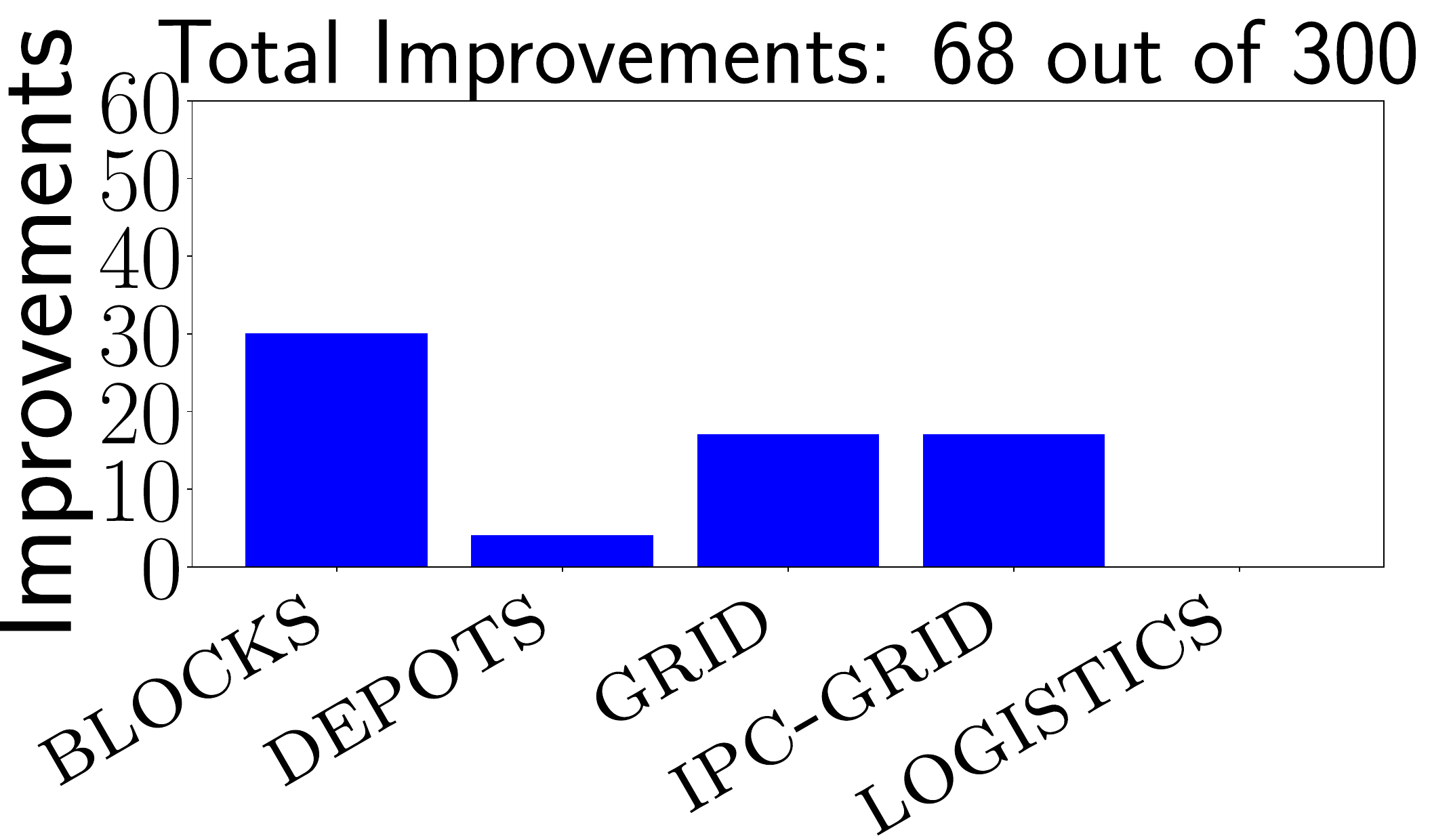}
        \caption{\it Plan Transparency.}
        \label{fig:result_plan_transparency}
    \end{subfigure}
    \hfill
    \begin{subfigure}[b]{0.22\textwidth}
        \centering
        \includegraphics[width=\columnwidth]{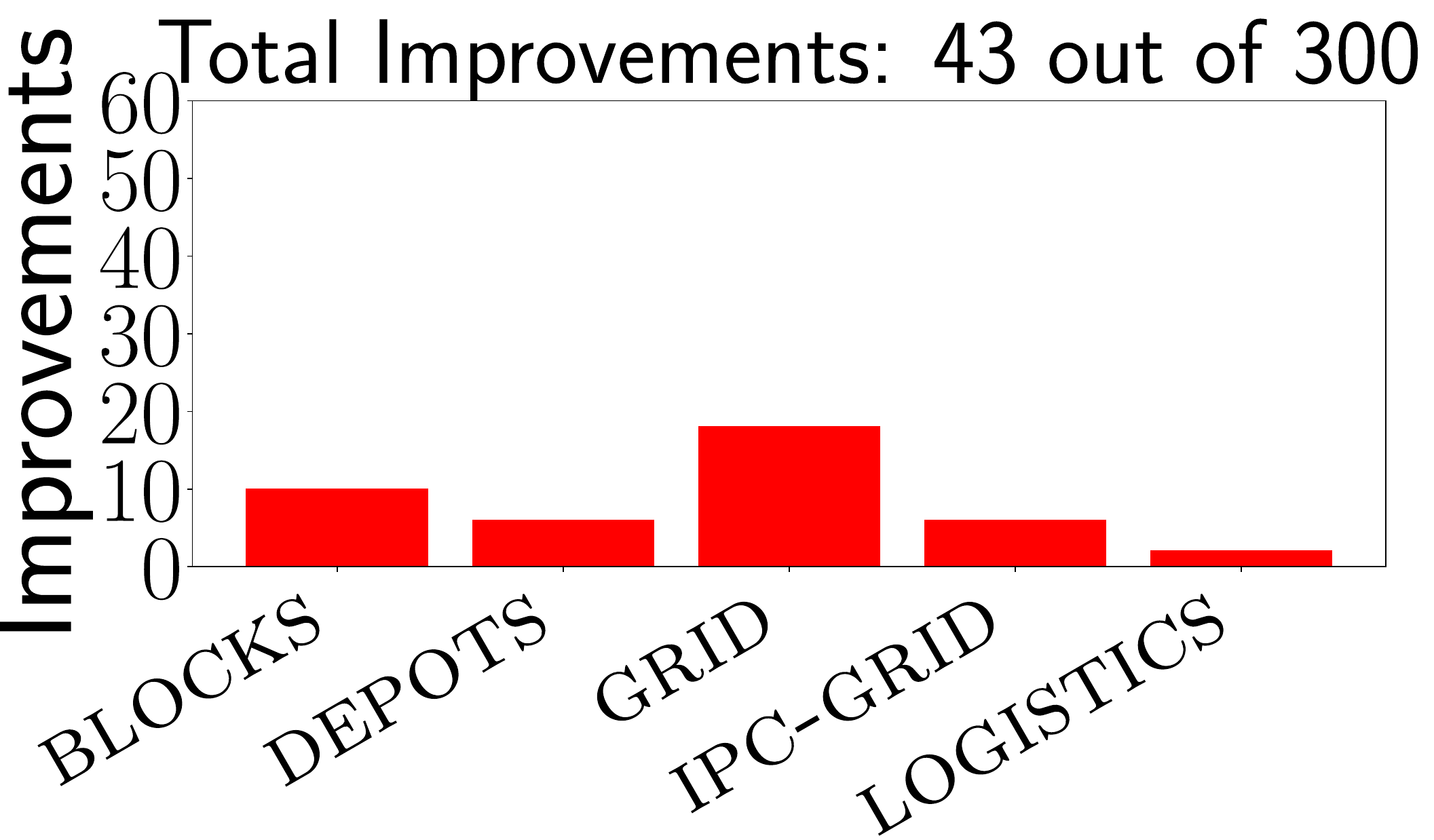}
        \caption{\it Goal Privacy.}
        \label{fig:result_goal_privacy}
    \end{subfigure}
    \hfill
    \begin{subfigure}[b]{0.22\textwidth}
        \centering
        \includegraphics[width=\columnwidth]{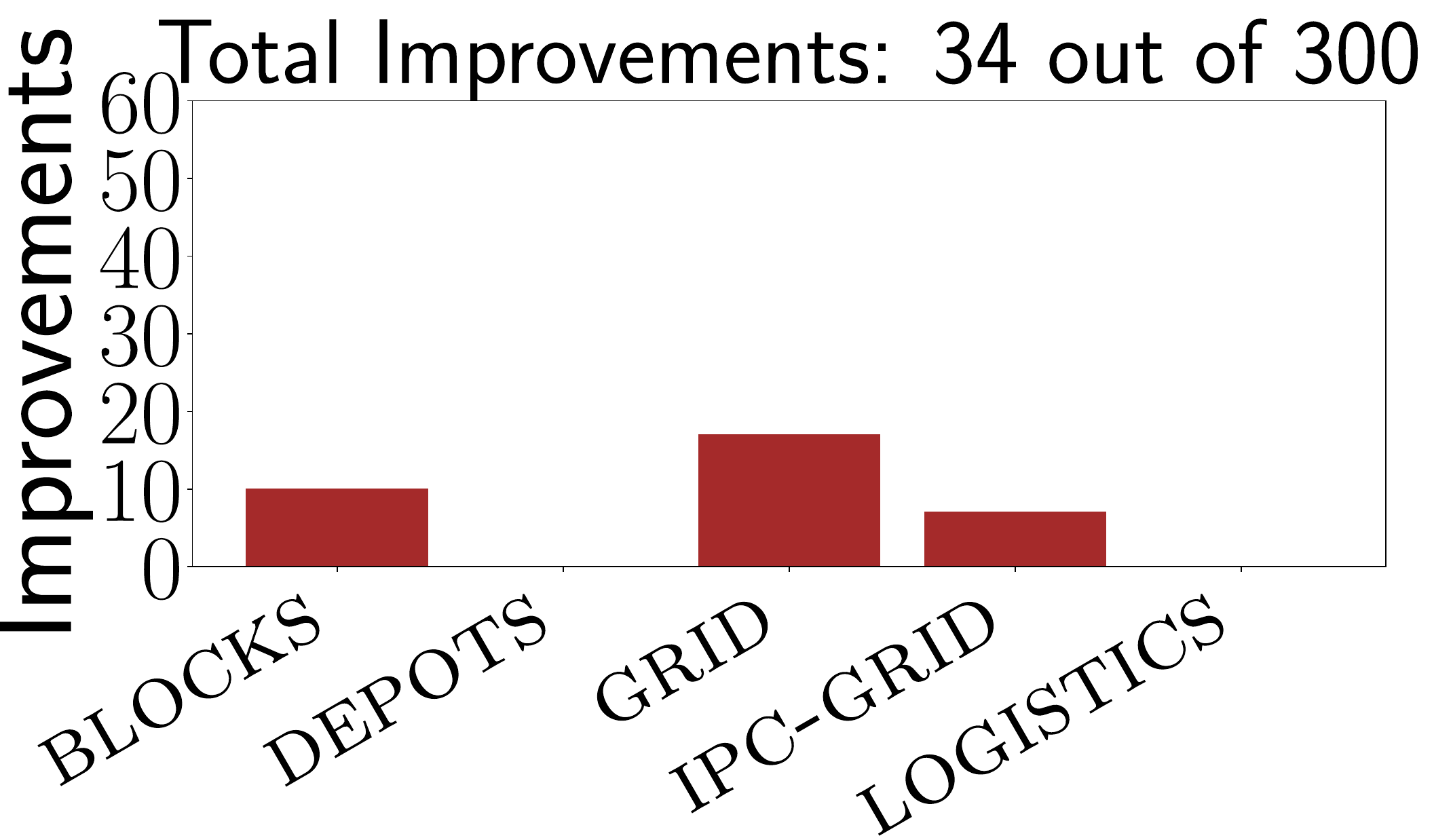}
        \caption{\it Plan Privacy.}
        \label{fig:result_plan_privacy}
    \end{subfigure}
    \\
    \begin{subfigure}[b]{0.22\textwidth}
        \centering
        \includegraphics[width=\columnwidth]{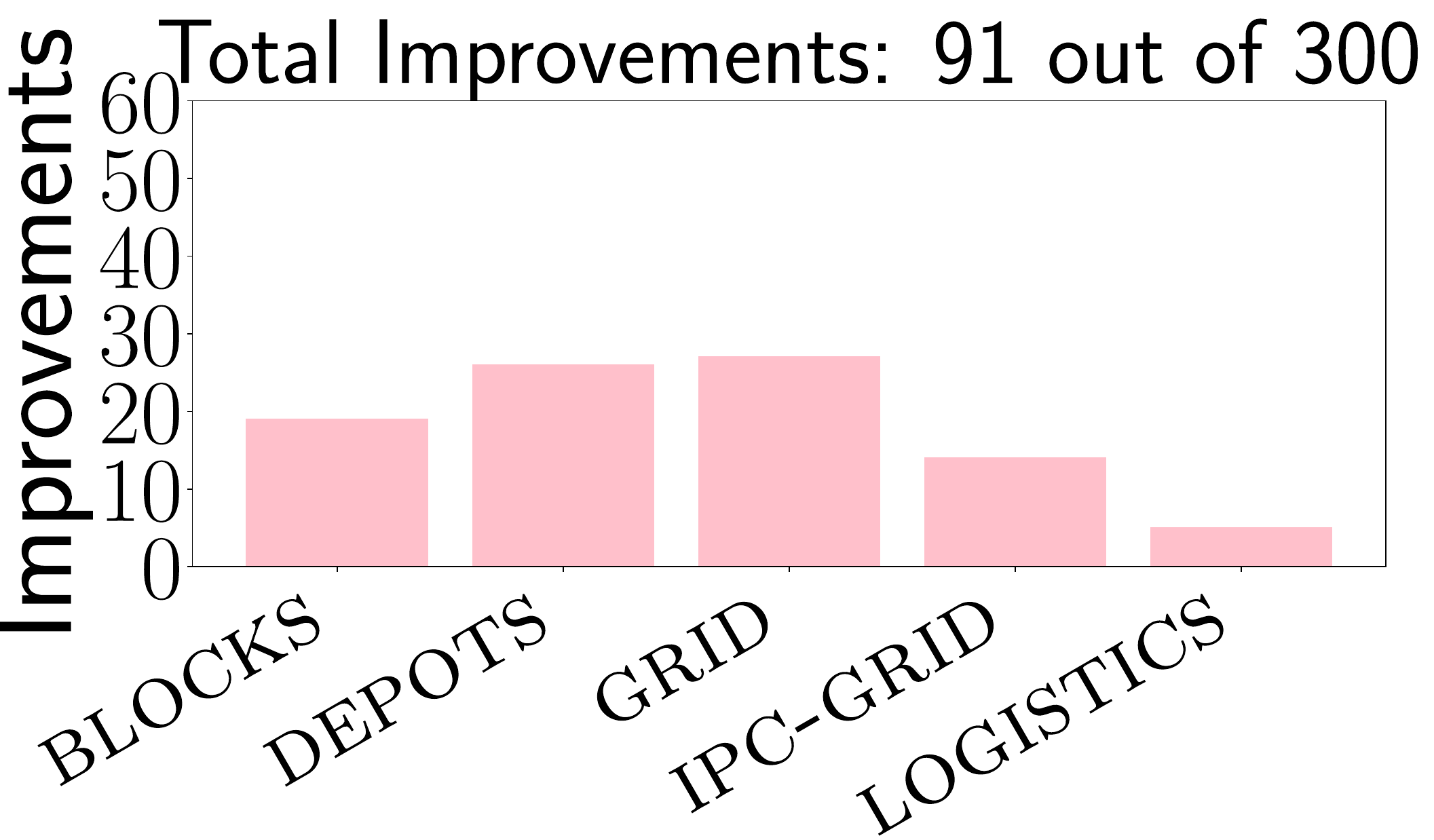}
        \caption{\it Min. Avg. Distance.}
        \label{fig:result_minAvgD}
    \end{subfigure}
    \hfill
    \begin{subfigure}[b]{0.22\textwidth}
        \centering
        \includegraphics[width=\columnwidth]{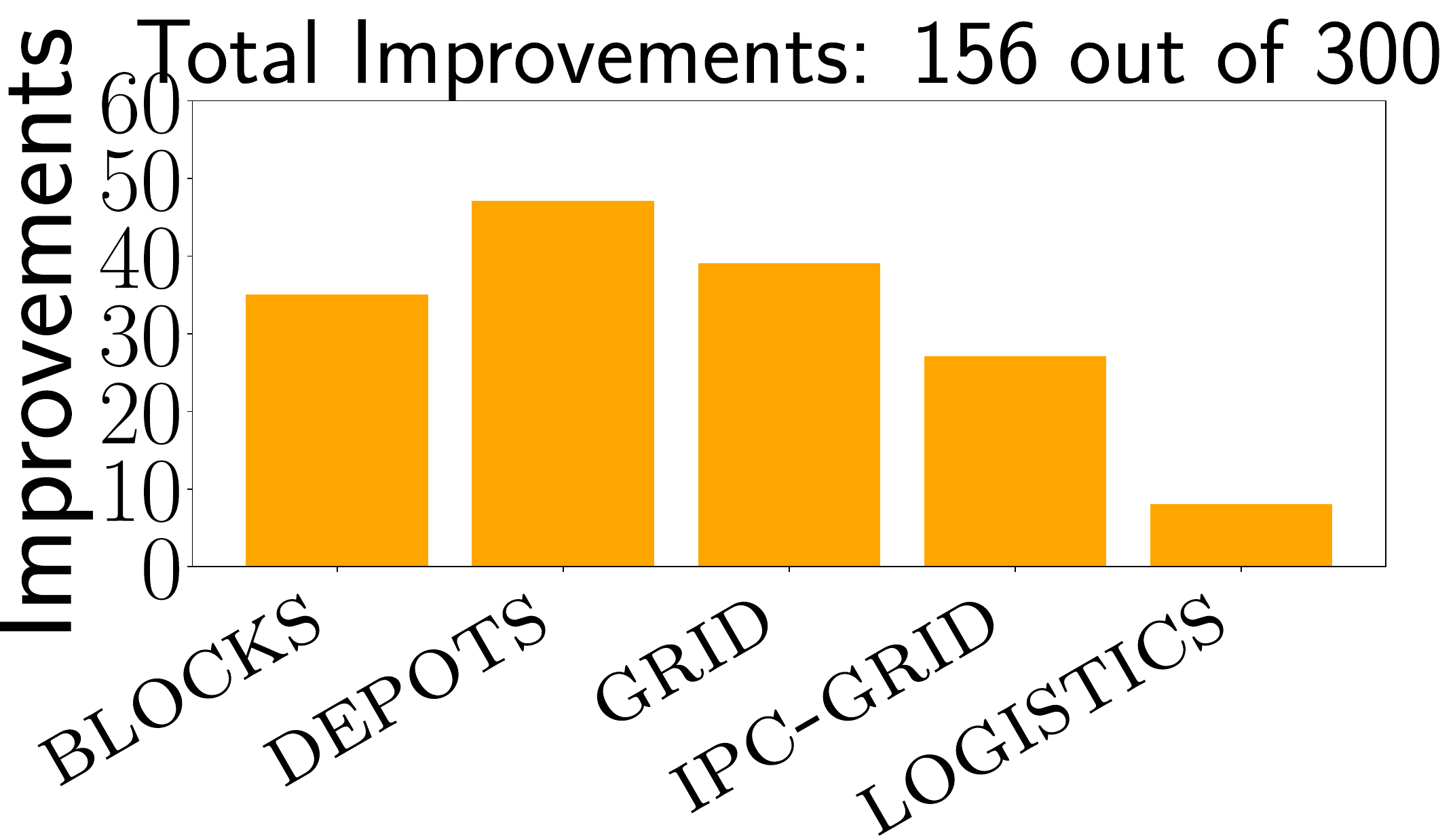}
        \caption{\it Max. Avg. Distance.}
        \label{fig:result_maxAvgD}
    \end{subfigure}
    \hfill
    \begin{subfigure}[b]{0.22\textwidth}
        \centering
        \includegraphics[width=\columnwidth]{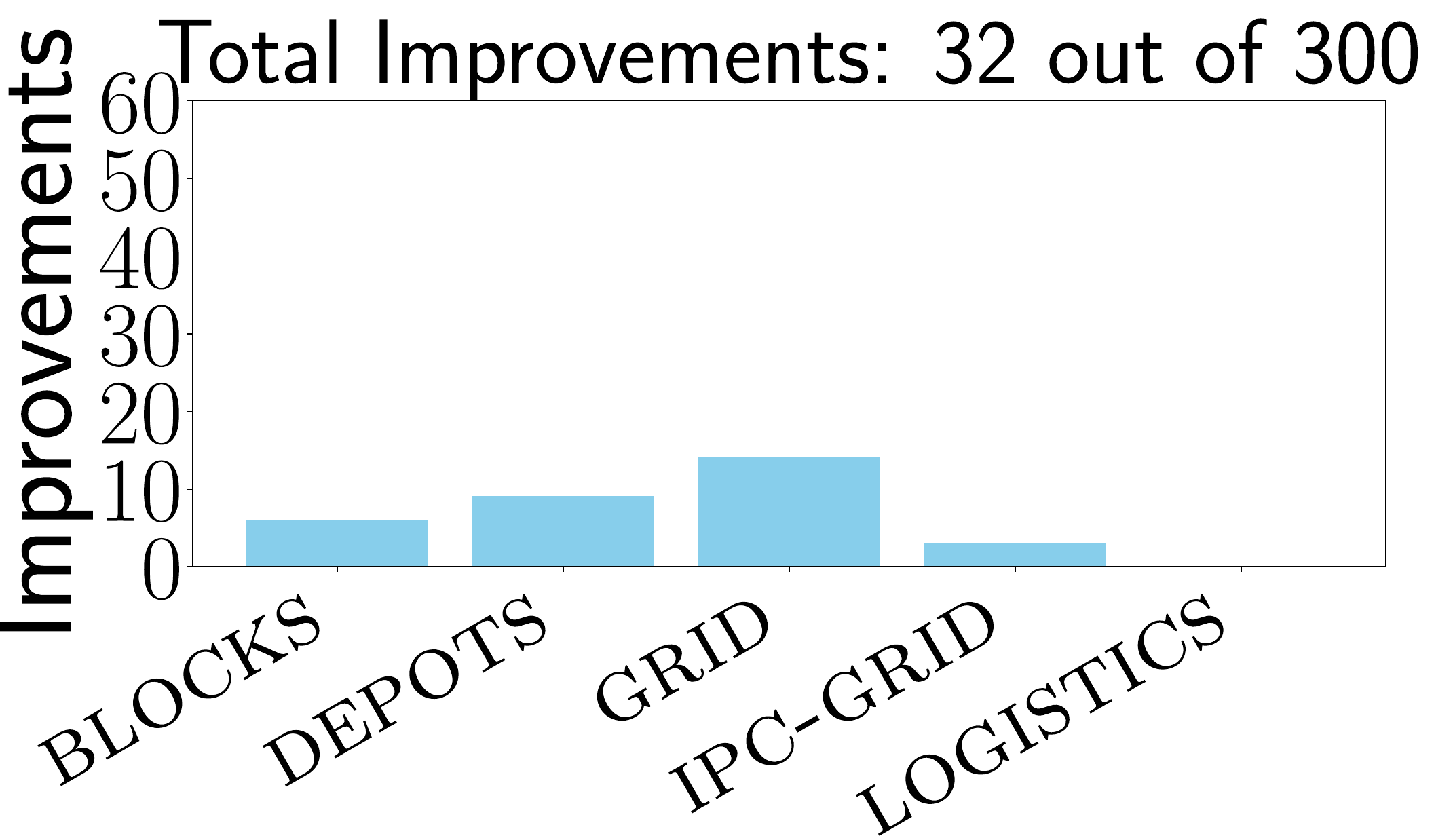}
        \caption{\it Min. Max. Distance.}
        \label{fig:result_minMaxD}
    \end{subfigure}
    \hfill
    \begin{subfigure}[b]{0.22\textwidth}
        \centering
        \includegraphics[width=\columnwidth]{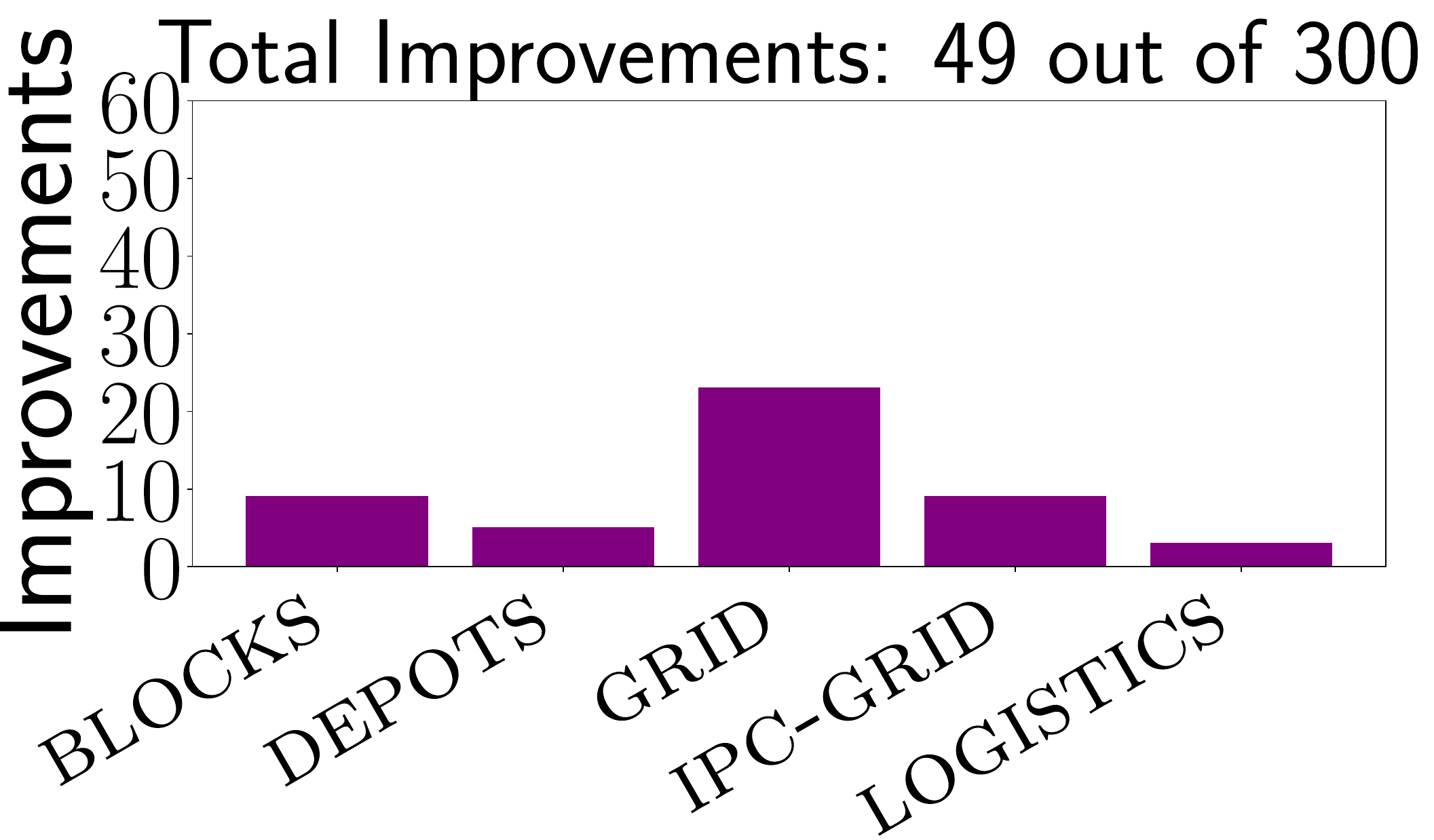}
        \caption{\it Max. Min. Distance.}
        \label{fig:result_maxMinD}
    \end{subfigure}
    \caption{Histograms for the redesign metrics, showing the number of improved/optimised problems using \approach.}
    \label{fig:histograms_improved_problems}
\end{figure*}

\begin{figure*}[h!]
    \centering
    \begin{subfigure}[b]{0.22\textwidth}
        \centering
        \includegraphics[width=\columnwidth]{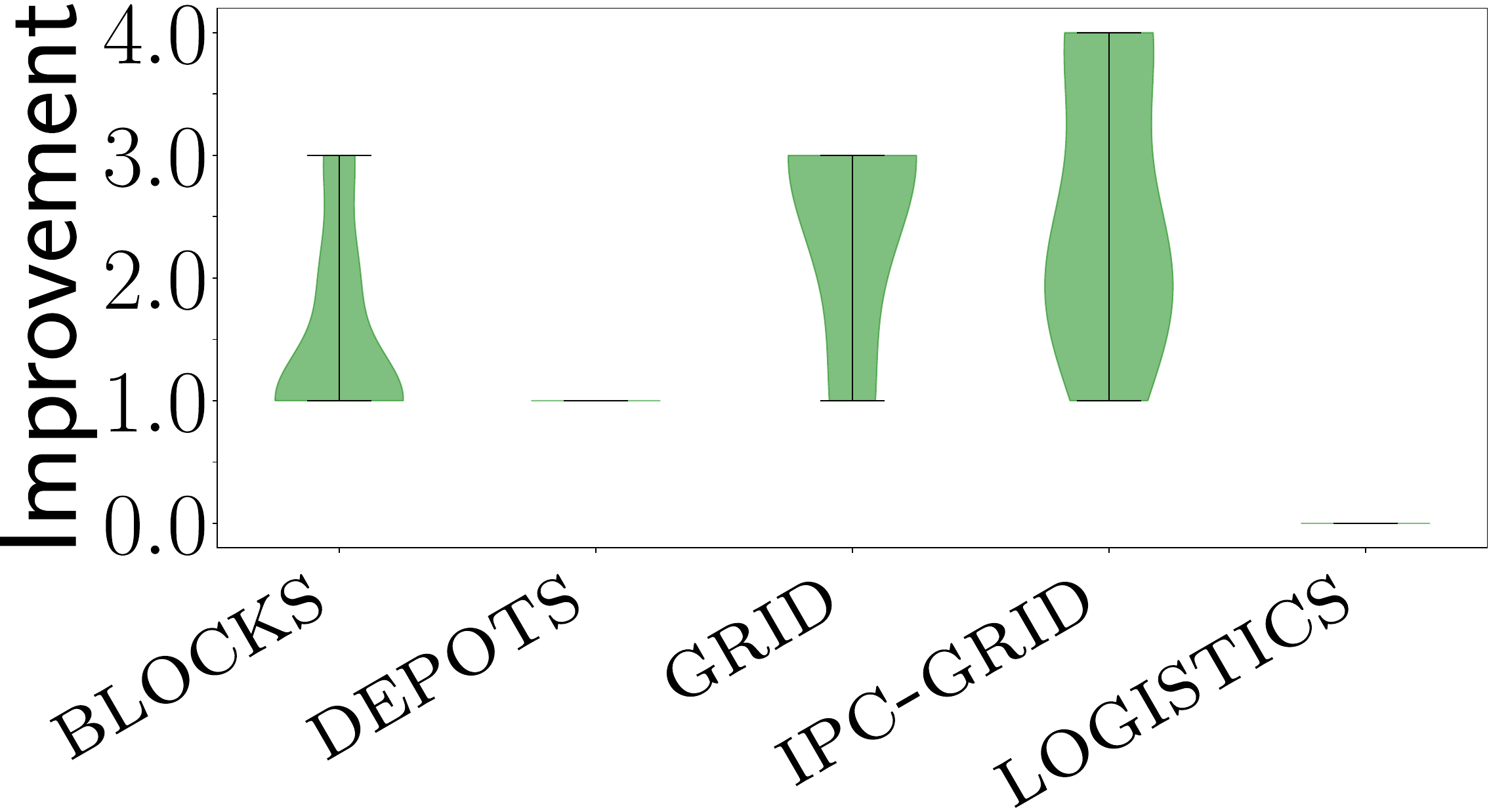}
        \caption{\it Goal Transparency.}
        \label{fig:goal_transparency_improvement}
    \end{subfigure}
    \hfill
    \begin{subfigure}[b]{0.22\textwidth}
        \centering
        \includegraphics[width=\columnwidth]{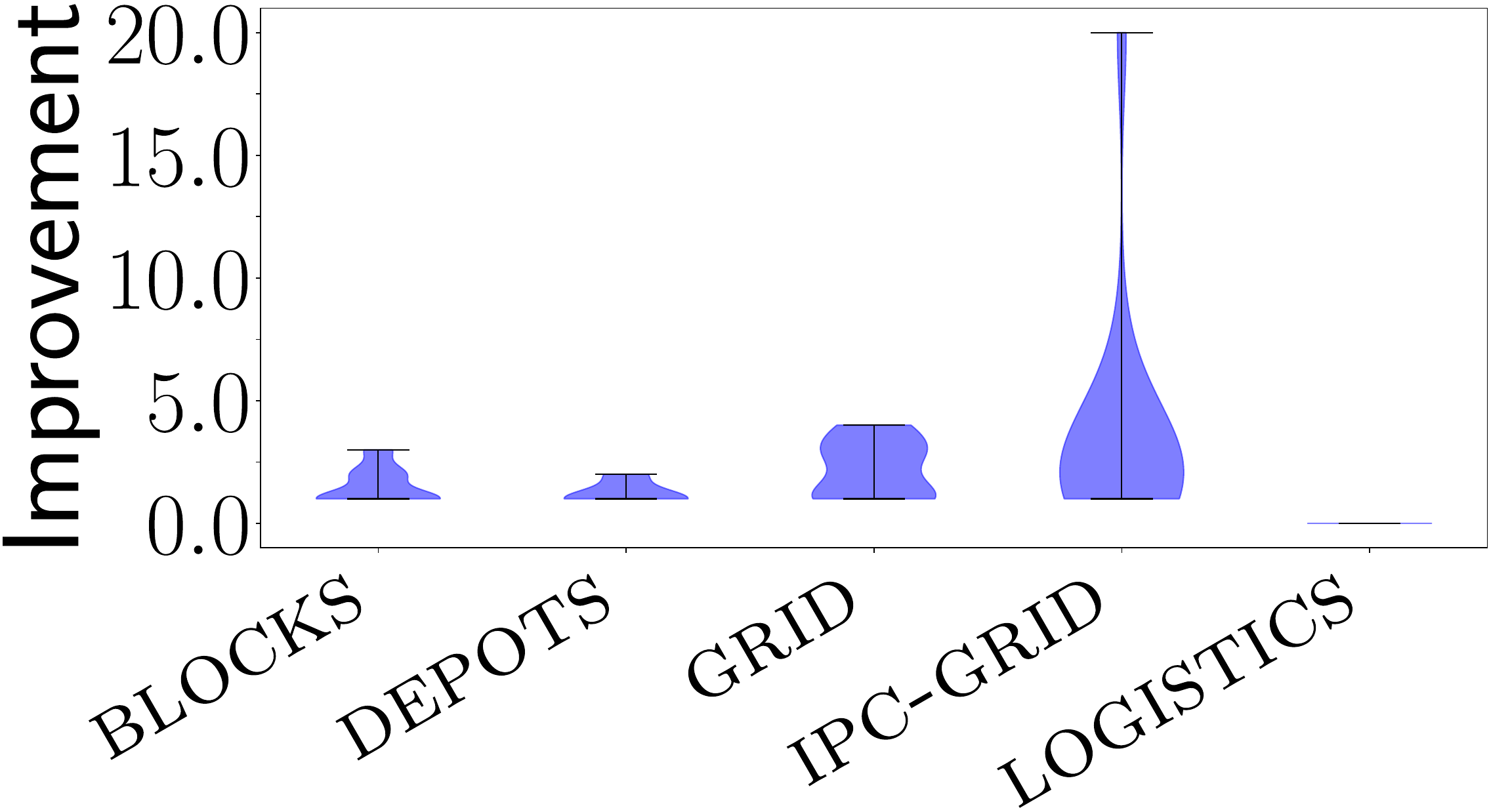}
        \caption{\it Plan Transparency.}
        \label{fig:plan_transparency_improvement}
    \end{subfigure}
    \hfill
    \begin{subfigure}[b]{0.22\textwidth}
        \centering
        \includegraphics[width=\columnwidth]{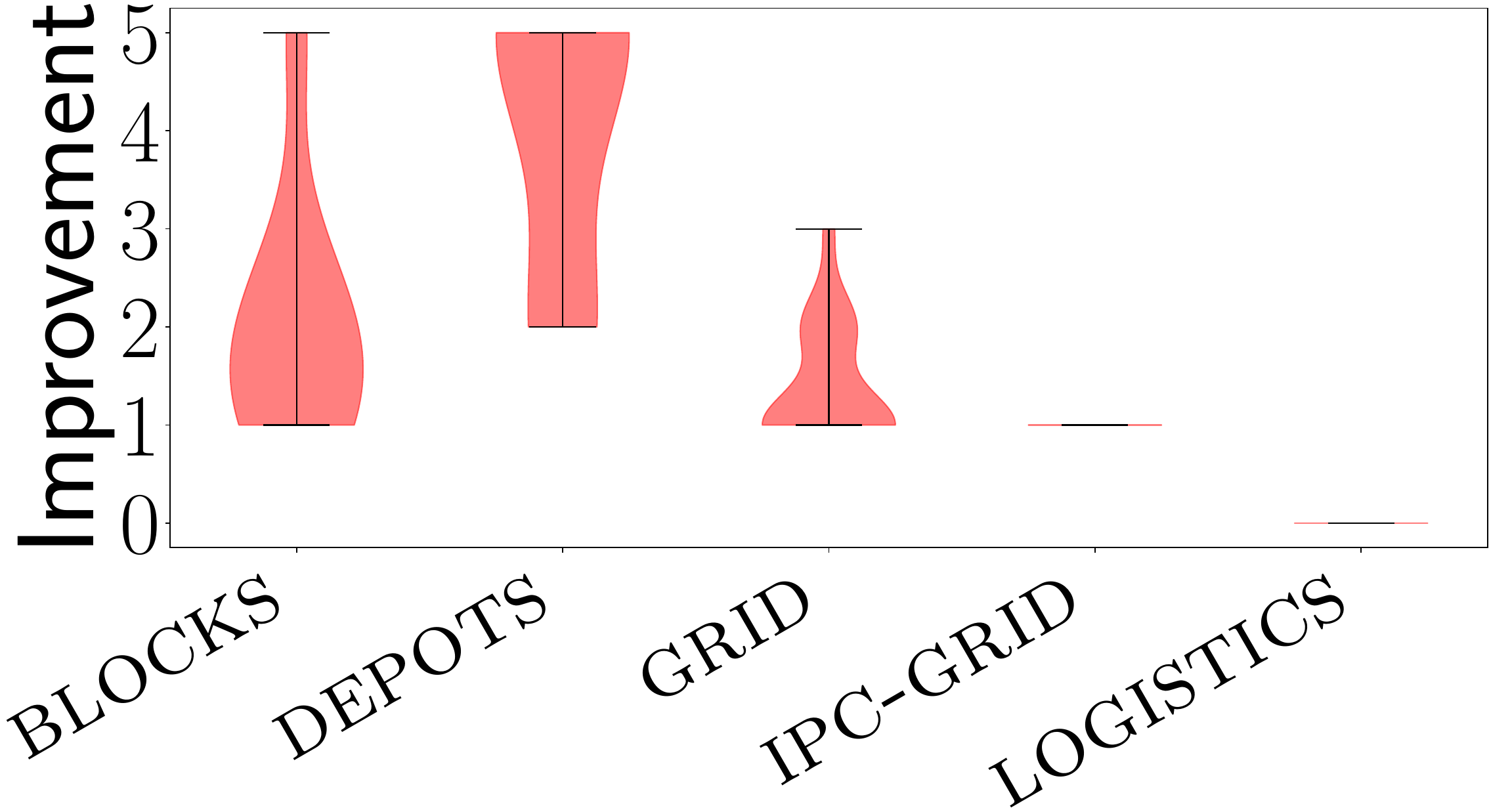}
        \caption{\it Goal Privacy.}
        \label{fig:goal_privacy_improvement}
    \end{subfigure}
    \hfill
    \begin{subfigure}[b]{0.22\textwidth}
        \centering
        \includegraphics[width=\columnwidth]{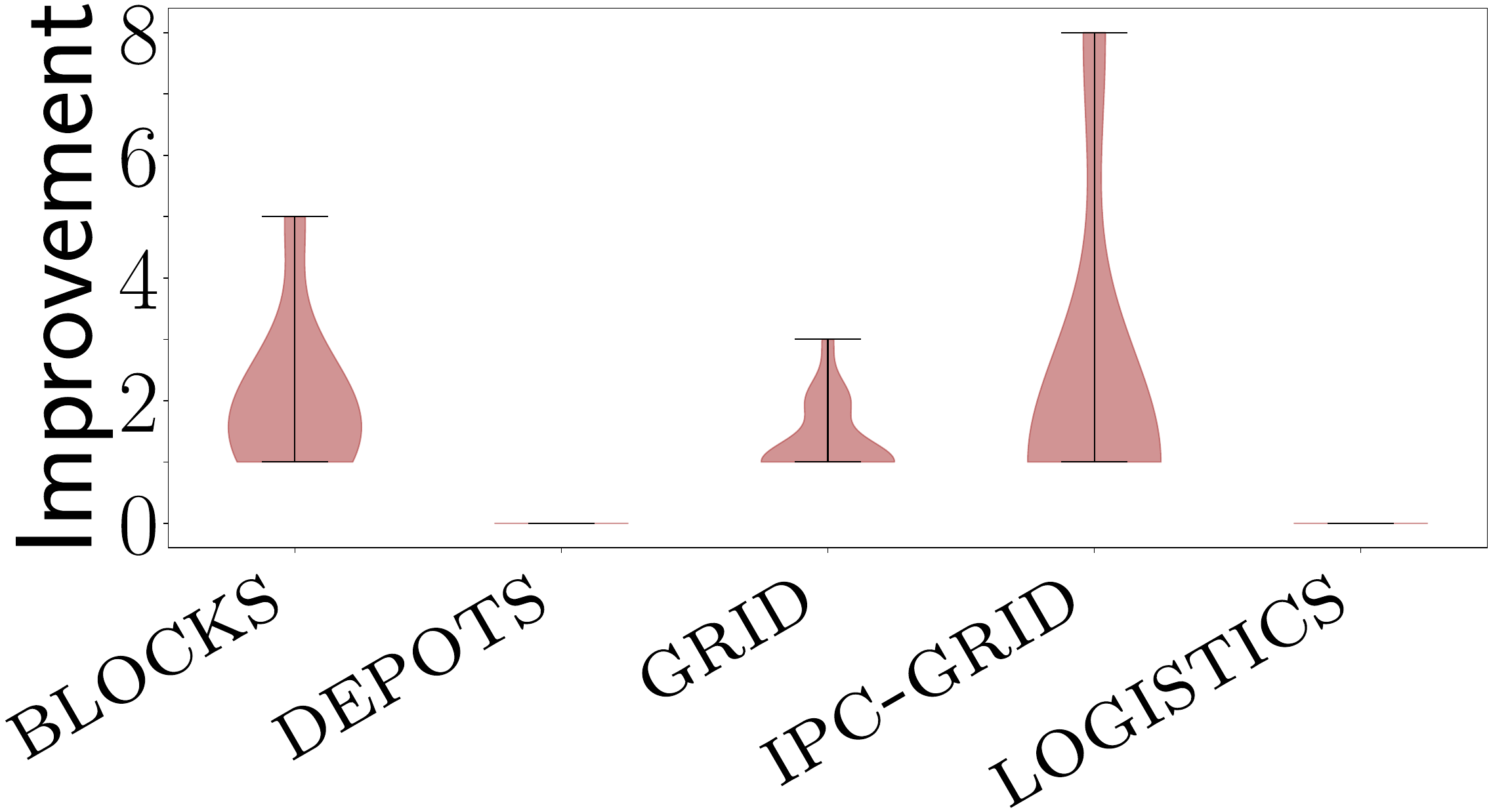}
        \caption{\it Plan Privacy.}
        \label{fig:plan_privacy_improvement}
    \end{subfigure}
    \\
    \begin{subfigure}[b]{0.22\textwidth}
        \centering
        \includegraphics[width=\columnwidth]{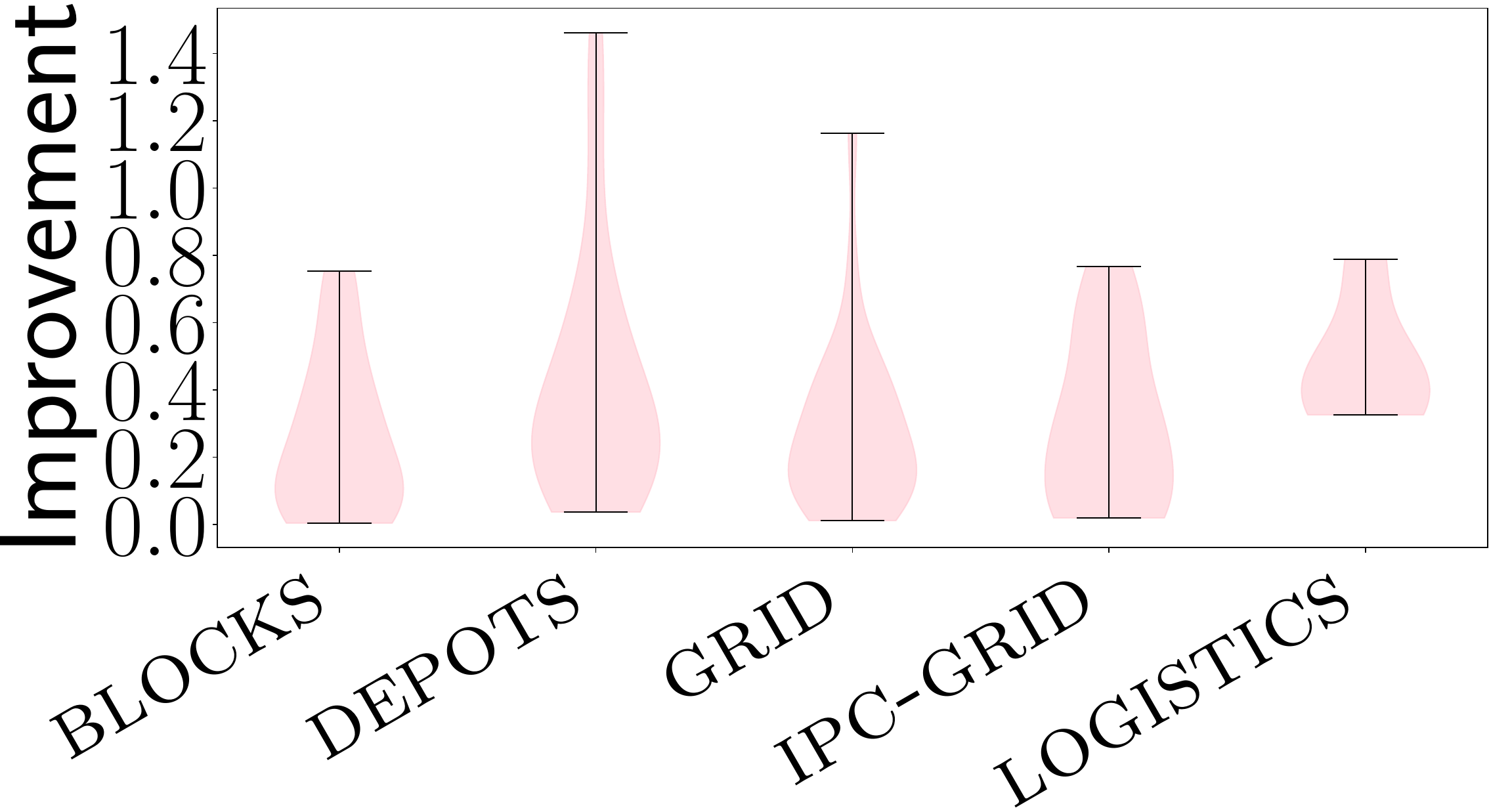}
        \caption{\it Min. Avg. Distance.}
        \label{fig:minAvgD_Improvement}
    \end{subfigure}
    \hfill
    \begin{subfigure}[b]{0.22\textwidth}
        \centering
        \includegraphics[width=\columnwidth]{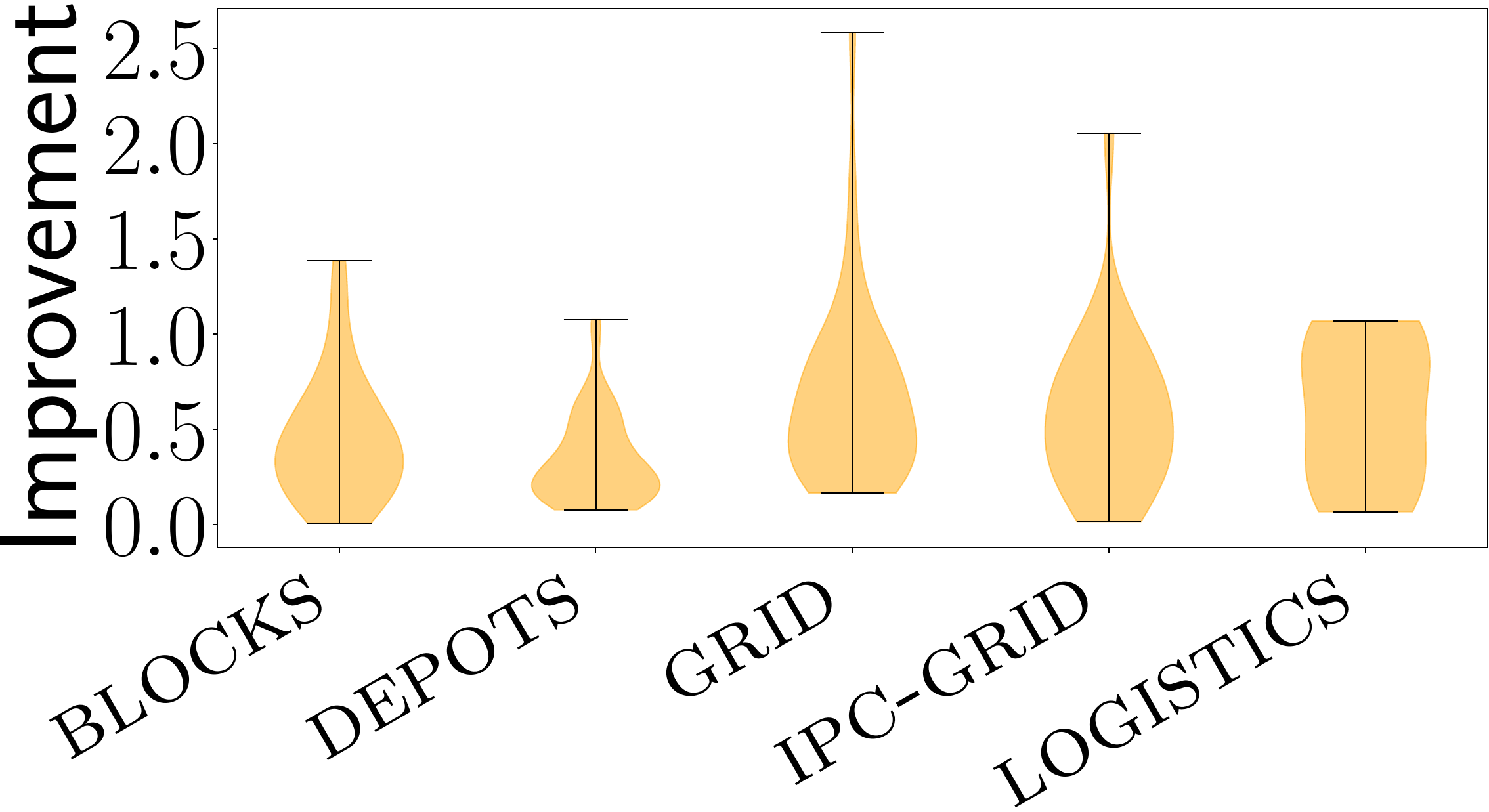}
        \caption{\it Max. Avg. Distance.}
        \label{fig:maxAvgD_Improvement}
    \end{subfigure}
    \hfill
    \begin{subfigure}[b]{0.22\textwidth}
        \centering
        \includegraphics[width=\columnwidth]{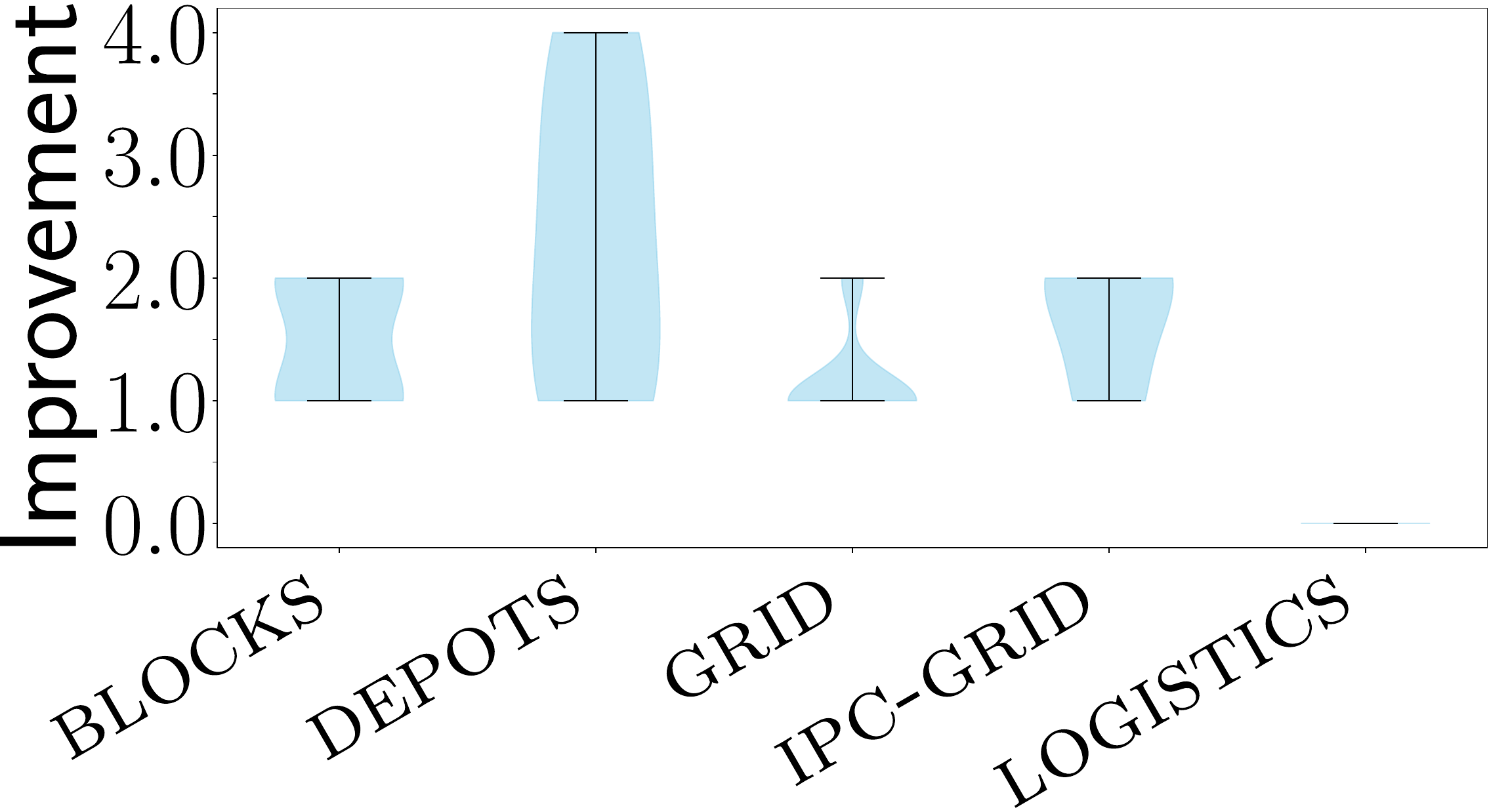}
        \caption{\it Min. Max. Distance.}
        \label{fig:minMaxD_Improvement}
    \end{subfigure}
    \hfill
    \begin{subfigure}[b]{0.22\textwidth}
        \centering
        \includegraphics[width=\columnwidth]{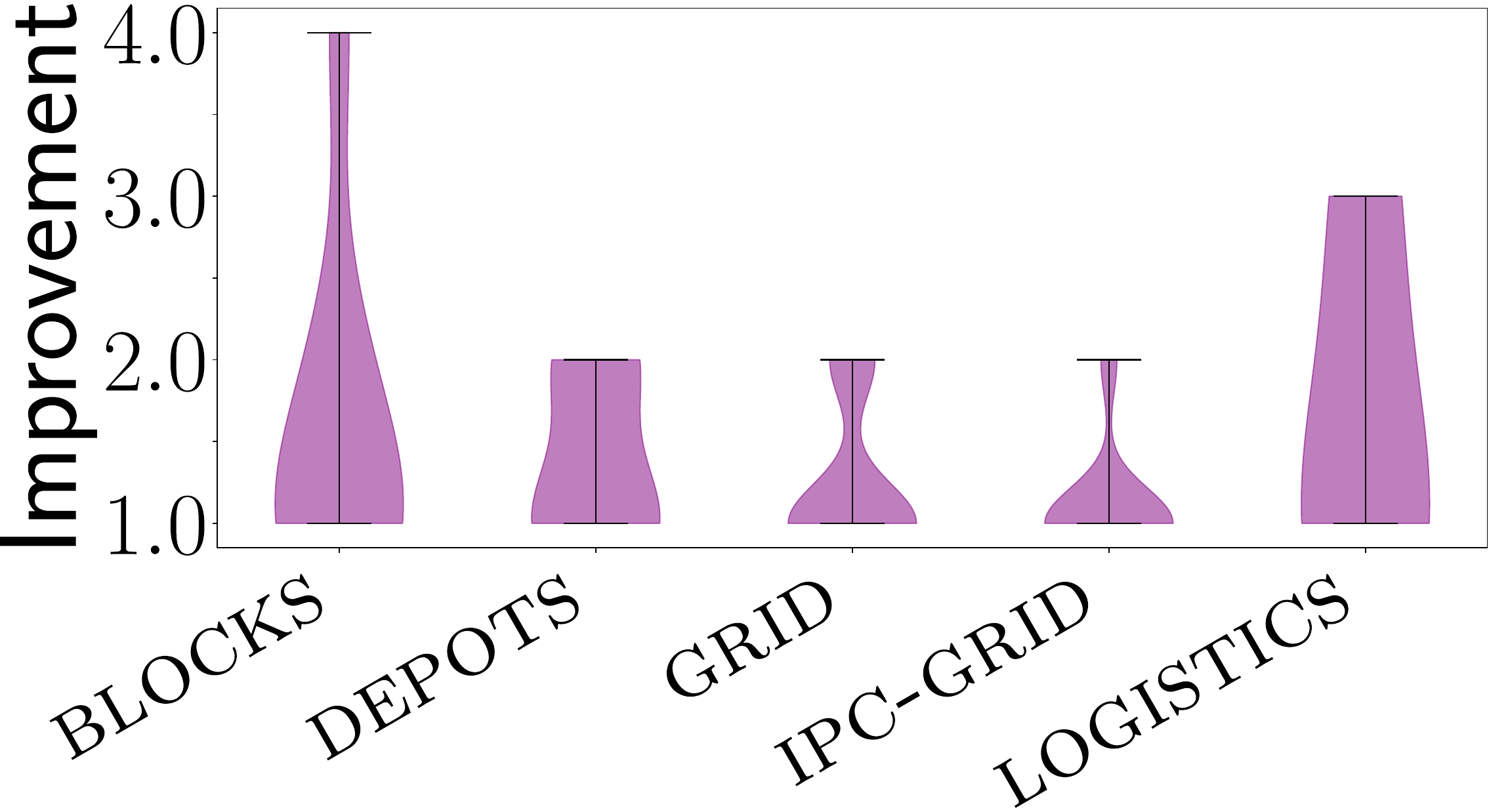}
        \caption{\it Max. Min. Distance.}
        \label{fig:maxMinD_Improvement}
    \end{subfigure}
    \hfill    
    \caption{Violin plots for the redesign metrics, showing the improvement from $m_0$ to $m^{+}$ ($|m_0 - m^{+}|$) using \approach.}
    \label{fig:histograms_improvement_mzero_mstar}
\end{figure*} 

As we can see in the GT columns (first inner table), 
our approach \approach yields the same results (same redesigned environments) as \grd-\textit{LS} \cite{ICAPS_KerenGK14} but two orders of magnitude faster.
\grd-\textit{LS} needs $119.9$ seconds on average to find the best solution in $8$ out of $60$ \ipcgrid~problems for which it returns improved environment. Instead, \approach only needs $1.1$ seconds on average.
This performance gap can be explained by two factors.
First, \approach uses \textsc{sym-k} to compute a plan-library before searching in the space of environment's modifications. By only removing the actions appearing in this library, \approach needs to explore much fewer nodes than \grd-\textit{LS}.
Second, \grd-\textit{LS} needs to generate and solve new planning problems for each node in order to compute the \textit{wcd} of the new environment, resulting in a huge computational overhead. 
On the contrary, \approach can compute the \textit{wcd} very efficiently by just analysing the common prefixes of the plans in the plan-library.

\approach's execution time takes into account the plan-library computation and the BFS search, depending on the domain and given planning environment.
In \blocks, computing $\planlibrary$ using \textsc{sym-k} takes most of the execution time, an average of $1.0$ seconds, whereas the search only takes $0.02$ seconds.
Most \blocks~problems only have a few optimal plans to achieve the goals $\goals$, therefore pruning the space of actions' removal, promoting efficient search.
In the problems that have many optimal plans (e.g., \gridnavigation~problems), \approach spends all the time in computing the plan-library, having no time for the search.
We note that in other domains, or when redesigning environments to optimise the distance-based metrics, the search procedure takes most of the execution time.

\approach's execution time increases when redesigning environments to optimise the distance-based metrics.
Two factors influence this: (1) the space of action's removal is larger, as \approach is not constrained to only remove actions in the plan-library; and (2) evaluating the metric of each search state is more costly than for the other metrics. For GT, PT, GP, and PP, \approach only reasons over plans and their common prefixes, while for the distance-based metrics \approach computes the optimal costs from each state traversed by optimal plans that achieve $G_t$ and the other states in $\goals_S = \goals \setminus \{G_t\}$.

\subsubsection{How \approach Improves Planning Environments.}

Figures~\ref{fig:histograms_improved_problems}~and~\ref{fig:histograms_improvement_mzero_mstar} show the results of \approach for the number of improved problems and the actual improvement from the initial metric value $m_0$ to the best metric value found $m^{+}$ (i.e., $|m_0 - m^{+}|$).
As we can see in Figures~\ref{fig:result_goal_transparency},~\ref{fig:result_plan_transparency},~\ref{fig:result_goal_privacy}, and~\ref{fig:result_plan_privacy}, our approach \approach improved/optimised the initial metric $m_0$ for approximately $16\%$ of the problems in our datasets for the \textit{transparency} and \textit{privacy} metrics (within the 900 seconds time limit). In contrast, \approach was able to improve/optimise many more problems for the distance-related metrics (see Figures~\ref{fig:result_minAvgD},~\ref{fig:result_maxAvgD},~\ref{fig:result_minMaxD}, and~\ref{fig:result_maxMinD}), i.e., achieving $27.5\%$ of improved/optimised problems. 
However, \approach has encountered considerable more difficulty in improving/optimising problems for the \textit{transparency} and \textit{privacy} metrics, especially in domains and problems with more actions. These problems involve more complex and bigger state spaces with more combinations of possible actions and (optimal) plans (i.e., \depots and \logistics). 
The number of possible goals, ``where they are placed'', and the distance among them could also affect the environment redesign process.
Redesigning environments for distance-related metrics tends to be less complex and difficult because there are cases where removing a single action (or few actions) may reduce/increase the distance among the goals, finding the set of removed actions at earlier stages in the search process.

Figure ~\ref{fig:histograms_improvement_mzero_mstar} depicts violin plots to show the actual improvement between the initial metric value $m_0$ and the best metric value found $m^{+}$, i.e., $|m_0 - m^{+}|$.
Overall, \approach could find significant improvements from the initial metric $m_0$ for the problems it was able to finish the redesign search process (within the time limit).
\approach has excelled in finding improvements for the distance-related metrics in most domains (see Figures~\ref{fig:maxAvgD_Improvement},~\ref{fig:maxAvgD_Improvement},~\ref{fig:minMaxD_Improvement}, and~\ref{fig:maxMinD_Improvement}), whereas it has faced considerable difficulties for finding improvements for the \textit{transparency} and \textit{privacy} metrics (see Figures~\ref{fig:goal_transparency_improvement},~\ref{fig:plan_transparency_improvement},~\ref{fig:goal_privacy_improvement}, and~\ref{fig:plan_privacy_improvement}), especially for \depots and \logistics~(as described previouly).

The number of solutions depends on the domain, planning problem, the metric, and the time and memory limits.
Across all the problems for which \approach finds a metric improvement, it returns an average of $14.35$ solutions. There are problems where \approach returns up to $960$ different solutions.

\section*{Related Work}\label{sec:related_work}

This paper's contributions relate to several previous works in the literature in two different dimensions: the environment redesign objective and metric, and on the algorithmic side.

Most approaches to planning environment redesign assume the observer's objective is to modify the environment to facilitate the recognition of goals and plans~\cite{ICAPS_KerenGK14,GRD_SonSSSY16,TIST_MirskyGSK19}.
Later works in \grd frame and solve this task under different observability settings~\cite{Keren2015goal,GRD_AAAI_KerenGK16,IJCAI_KerenGK16}, environment assumptions~\cite{GRD_WayllaceH0S16,GRD_WayllaceH017,SGRD_WayllaceKGK0Z20,GRD_Wayllace022}, or observer's capabilities~\cite{AGR_ShvoM20,AGRD_GallRK21}.
Unlike these works, we assume the interested party might want to modify the environment for tasks different than recognising goals and plans.
The redesigned environments obtained when optimising our new metrics can be useful for many planning applications such as \textit{Counterplanning}~\cite{PozancoCounterPlanningEFB18} or \textit{Anticipatory Planning}~\cite{DBLP:conf/aips/BurnsBRYD12,aicomm18-anticipatory}, among others.

On the algorithmic side, most works use search algorithms to explore the space of actions' removal~\cite{keren2021goal}.
While \approach searches in the same space using similar algorithms, it differs from these works as follows.
First, \approach~presents a good compromise between approaches that do not use plan-libraries~\cite{JARI_KerenGK19} and those that need pre-defined plan-libraries~\cite{TIST_MirskyGSK19}.
\approach~exploits recent advances in top-quality planning to efficiently compute plan-libraries for pruning the space of modifications.
Second, \approach~is \textit{metric-agnostic}.
Previous approaches are \textit{metric-dependent}, devising pruning techniques and stopping conditions tailored to specific metrics, whereas \approach is more general and can accommodate a wide variety of metrics.
Third, \approach~is able to return all the best solutions found until a stopping condition is met.
This is usually a desirable feature in applications with \textit{humans-in-the-loop}~\cite{boddy2005course,sohrabi2018ai}, as humans prefer to have diverse solutions to choose from.

\section*{Conclusions}\label{sec:conclusions}

In this paper, we extended the definition of environment design from previous work~\cite{ICAPS_KerenGK14,JARI_KerenGK19,TIST_MirskyGSK19}, and we introduced a more general task for \textit{Planning Environment Redesign}. 
We defined a new set of environment redesign metrics that endows and facilitates not only the recognition of goals and plans (\textit{transparency} or \textit{predicability}), but also other tasks, such as goal and plan privacy, deception, risk avoidance, or plan for opportunities.
We showed that our general environment redesign approach \approach~is \textit{metric-agnostic}, and can optimise a wide variety of redesign metrics. Our experiments show that \approach~is efficient to optimise different metrics, and it outperforms (being orders of magnitude faster) the most efficient \grd approach of \citet{ICAPS_KerenGK14}.

We intend to expand this work in two directions: improving \approach's performance and making the problem definition even more general.
Regarding \approach's performance, we aim to develop heuristics to improve the redesign search process.
Namely, one could prioritise removing actions belonging to a higher number of plans in the plan-library when optimising goal transparency. 
Moreover, we aim to study how to balance the time allocated to compute the plan-library and the time to search in the space of modifications.
While in some problems we may be interested in computing the whole plan-library to ensure optimality, in others this task might take a large amount of time. Then, we could decide to only compute a subset of plans that would then guide the search rather than pruning actions not appearing on it.
As for the problem formulation, the only environment modifications allowed by most approaches are actions' removal. We envisage allowing other modifications such as removing or adding objects or predicates from the initial state.
Finally, we aim to investigate how to jointly optimize sets of these metrics by framing planning environment redesign as a multi-objective task.

\section*{Acknowledgements} 

This paper was prepared for informational purposes in part by the Artificial Intelligence Research group of J.P. Morgan Chase \& Co and its affiliates (J.P. Morgan) and is not a product of the Research Department of J.P. Morgan. 
J.P. Morgan makes no representation and warranty whatsoever and disclaims all liability, for the completeness, accuracy, or reliability of the information contained herein. This document is not intended as investment research or investment advice, or a recommendation, offer, or solicitation for the purchase or sale of any security, financial instrument, financial product, or service, or to be used in any way for evaluating the merits of participating in any transaction. It shall not constitute a solicitation under any jurisdiction or to any person if such solicitation under such jurisdiction or to such person would be unlawful.

\bibliography{bibliography}

\begin{thebibliography}{41}
\providecommand{\natexlab}[1]{#1}

\bibitem[{Bernardini, Fagnani, and Franco(2020)}]{GoalObsf_BernardiniFF20}
Bernardini, S.; Fagnani, F.; and Franco, S. 2020.
\newblock An Optimization Approach to Robust Goal Obfuscation.
\newblock In \emph{KR}.

\bibitem[{Boddy et~al.(2005)Boddy, Gohde, Haigh, and Harp}]{boddy2005course}
Boddy, M.~S.; Gohde, J.; Haigh, T.; and Harp, S.~A. 2005.
\newblock Course of Action Generation for Cyber Security Using Classical Planning.
\newblock In \emph{ICAPS}.

\bibitem[{Borrajo and Veloso(2021{\natexlab{a}})}]{DBLP:conf/aips/BorrajoV21}
Borrajo, D.; and Veloso, M. 2021{\natexlab{a}}.
\newblock Computing Opportunities to Augment Plans for Novel Replanning during Execution.
\newblock In \emph{ICAPS}.

\bibitem[{Borrajo and Veloso(2021{\natexlab{b}})}]{iros21}
Borrajo, D.; and Veloso, M. 2021{\natexlab{b}}.
\newblock Intelligent Execution through Plan Analysis.
\newblock In \emph{Proceedings of IROS}.

\bibitem[{Burns et~al.(2012)Burns, Benton, Ruml, Yoon, and Do}]{DBLP:conf/aips/BurnsBRYD12}
Burns, E.; Benton, J.; Ruml, W.; Yoon, S.~W.; and Do, M.~B. 2012.
\newblock Anticipatory On-Line Planning.
\newblock In \emph{ICAPS}.

\bibitem[{Chakraborti et~al.(2019)Chakraborti, Kulkarni, Sreedharan, Smith, and Kambhampati}]{Chakraborti2019explicability}
Chakraborti, T.; Kulkarni, A.; Sreedharan, S.; Smith, D.~E.; and Kambhampati, S. 2019.
\newblock Explicability? Legibility? Predictability? Transparency? Privacy? Security? The Emerging Landscape of Interpretable agent Behavior.
\newblock In \emph{ICAPS}.

\bibitem[{Fuentetaja, Borrajo, and de~la Rosa(2018)}]{aicomm18-anticipatory}
Fuentetaja, R.; Borrajo, D.; and de~la Rosa, T. 2018.
\newblock Anticipation of Goals in Automated Planning.
\newblock \emph{AI Communications}, 31(2): 117--135.

\bibitem[{Gall, Ruml, and Keren(2021)}]{AGRD_GallRK21}
Gall, K.~C.; Ruml, W.; and Keren, S. 2021.
\newblock Active Goal Recognition Design.
\newblock In \emph{IJCAI}.

\bibitem[{Geffner and Bonet(2013)}]{GeffnerBonet13_PlanningBook}
Geffner, H.; and Bonet, B. 2013.
\newblock \emph{A Concise Introduction to Models and Methods for Automated Planning}.
\newblock Morgan \& Claypool Publishers.

\bibitem[{Helmert(2006)}]{FastDownward_Helmert06}
Helmert, M. 2006.
\newblock The Fast Downward Planning System.
\newblock \emph{Journal of Artificial Intelligence Research}, 26: 191--246.

\bibitem[{Karpas(2022)}]{Centroids_MCS_Karpas22}
Karpas, E. 2022.
\newblock A Compilation Based Approach to Finding Centroids and Minimum Covering States in Planning.
\newblock In \emph{ICAPS}.

\bibitem[{Katz, Sohrabi, and Udrea(2020)}]{katz2020top}
Katz, M.; Sohrabi, S.; and Udrea, O. 2020.
\newblock Top-quality planning: Finding practically useful sets of best plans.
\newblock In \emph{ICAPS}.

\bibitem[{Keren, Gal, and Karpas(2014)}]{ICAPS_KerenGK14}
Keren, S.; Gal, A.; and Karpas, E. 2014.
\newblock Goal Recognition Design.
\newblock In \emph{ICAPS}.

\bibitem[{Keren, Gal, and Karpas(2015)}]{Keren2015goal}
Keren, S.; Gal, A.; and Karpas, E. 2015.
\newblock Goal recognition design for non-optimal agents.
\newblock In \emph{AAAI}.

\bibitem[{Keren, Gal, and Karpas(2016{\natexlab{a}})}]{GRD_AAAI_KerenGK16}
Keren, S.; Gal, A.; and Karpas, E. 2016{\natexlab{a}}.
\newblock Goal Recognition Design with Non-Observable Actions.
\newblock In \emph{AAAI}.

\bibitem[{Keren, Gal, and Karpas(2016{\natexlab{b}})}]{IJCAI_KerenGK16}
Keren, S.; Gal, A.; and Karpas, E. 2016{\natexlab{b}}.
\newblock Privacy Preserving Plans in Partially Observable Environments.
\newblock In \emph{IJCAI}.

\bibitem[{Keren, Gal, and Karpas(2019)}]{JARI_KerenGK19}
Keren, S.; Gal, A.; and Karpas, E. 2019.
\newblock Goal Recognition Design in Deterministic Environments.
\newblock \emph{Journal of Artificial Intelligence Research}, 65: 209--269.

\bibitem[{Keren, Gal, and Karpas(2021)}]{keren2021goal}
Keren, S.; Gal, A.; and Karpas, E. 2021.
\newblock Goal Recognition Design - Survey.
\newblock In \emph{IJCAI}.

\bibitem[{Kulkarni et~al.(2020)Kulkarni, Sreedharan, Keren, Chakraborti, Smith, and Kambhampati}]{kulkarni2020designing}
Kulkarni, A.; Sreedharan, S.; Keren, S.; Chakraborti, T.; Smith, D.~E.; and Kambhampati, S. 2020.
\newblock Designing environments conducive to interpretable robot behavior.
\newblock In \emph{IROS}.

\bibitem[{Kulkarni, Srivastava, and Kambhampati(2019)}]{kulkarni2019unified}
Kulkarni, A.; Srivastava, S.; and Kambhampati, S. 2019.
\newblock A unified framework for planning in adversarial and cooperative environments.
\newblock In \emph{AAAI}.

\bibitem[{MacNally et~al.(2018)MacNally, Lipovetzky, Ram{\'{\i}}rez, and Pearce}]{ActionSelection_MacNallyLRP18}
MacNally, A.~M.; Lipovetzky, N.; Ram{\'{\i}}rez, M.; and Pearce, A.~R. 2018.
\newblock Action Selection for Transparent Planning.
\newblock In \emph{AAMAS}.

\bibitem[{Masters and Sardi{\~{n}}a(2017)}]{MastersDeceptionS17}
Masters, P.; and Sardi{\~{n}}a, S. 2017.
\newblock Deceptive Path-Planning.
\newblock In \emph{IJCAI}.

\bibitem[{McDermott et~al.(1998)McDermott, Ghallab, Howe, Knoblock, Ram, Veloso, Weld, and Wilkins}]{PDDLMcdermott1998}
McDermott, D.; Ghallab, M.; Howe, A.; Knoblock, C.; Ram, A.; Veloso, M.; Weld, D.; and Wilkins, D. 1998.
\newblock {PDDL} $-$ {The Planning Domain Definition Language}.
\newblock In \emph{AIPS}.

\bibitem[{Mirsky et~al.(2019)Mirsky, Gal, Stern, and Kalech}]{TIST_MirskyGSK19}
Mirsky, R.; Gal, K.; Stern, R.; and Kalech, M. 2019.
\newblock Goal and Plan Recognition Design for Plan Libraries.
\newblock \emph{{ACM} Transactions on Intelligent Systems and Technology}, 10(2): 14:1--14:23.

\bibitem[{Perny, Spanjaard, and Storme(2007)}]{perny2007state}
Perny, P.; Spanjaard, O.; and Storme, L.~S. 2007.
\newblock State space search for risk-averse agents.
\newblock In \emph{IJCAI}.

\bibitem[{Pozanco et~al.(2018)Pozanco, E{-}Mart{\'{\i}}n, Fern{\'{a}}ndez, and Borrajo}]{PozancoCounterPlanningEFB18}
Pozanco, A.; E{-}Mart{\'{\i}}n, Y.; Fern{\'{a}}ndez, S.; and Borrajo, D. 2018.
\newblock Counterplanning using Goal Recognition and Landmarks.
\newblock In \emph{IJCAI}.

\bibitem[{Pozanco et~al.(2019)Pozanco, E{-}Mart{\'{\i}}n, Fern{\'{a}}ndez, and Borrajo}]{Centroids_MCS_PozancoEFB19}
Pozanco, A.; E{-}Mart{\'{\i}}n, Y.; Fern{\'{a}}ndez, S.; and Borrajo, D. 2019.
\newblock Finding Centroids and Minimum Covering States in Planning.
\newblock In \emph{ICAPS}.

\bibitem[{Pozanco et~al.(2022)Pozanco, E{-}Mart{\'{\i}}n, Fern{\'{a}}ndez, and Borrajo}]{arxiv22-counterplanning}
Pozanco, A.; E{-}Mart{\'{\i}}n, Y.; Fern{\'{a}}ndez, S.; and Borrajo, D. 2022.
\newblock Anticipatory Counterplanning.
\newblock \emph{arXiv e-prints}.

\bibitem[{Pozanco et~al.(2020)Pozanco, E-Martín, Fernández, and Borrajo}]{IntexPozanco20}
Pozanco, A.; E-Martín, Y.; Fernández, S.; and Borrajo, D. 2020.
\newblock Get me to Safety! Escaping from Risks using Automated Planning.
\newblock In \emph{ICAPS Intex/GR Workshop}.

\bibitem[{Price et~al.(2023)Price, Pereira, Masters, and Vered}]{PriceEtAl_AAMAS23}
Price, A.; Pereira, R.~F.; Masters, P.; and Vered, M. 2023.
\newblock Domain-Independent Deceptive Planning.
\newblock In \emph{AAMAS}.

\bibitem[{Ram{\'{\i}}rez and Geffner(2009)}]{RamirezGeffnerIJCAI2009}
Ram{\'{\i}}rez, M.; and Geffner, H. 2009.
\newblock {Plan Recognition as Planning}.
\newblock In \emph{{IJCAI}}.

\bibitem[{Russell and Norvig(2005)}]{Rssell2005ai}
Russell, S.; and Norvig, P. 2005.
\newblock AI a Modern Approach.
\newblock \emph{Learning}, 2(3): 4.

\bibitem[{Shvo and McIlraith(2020)}]{AGR_ShvoM20}
Shvo, M.; and McIlraith, S.~A. 2020.
\newblock Active Goal Recognition.
\newblock In \emph{AAAI}.

\bibitem[{Sohrabi et~al.(2018)Sohrabi, Riabov, Katz, and Udrea}]{sohrabi2018ai}
Sohrabi, S.; Riabov, A.; Katz, M.; and Udrea, O. 2018.
\newblock An AI planning solution to scenario generation for enterprise risk management.
\newblock In \emph{AAAI}.

\bibitem[{Son et~al.(2016)Son, Sabuncu, Schulz{-}Hanke, Schaub, and Yeoh}]{GRD_SonSSSY16}
Son, T.~C.; Sabuncu, O.; Schulz{-}Hanke, C.; Schaub, T.; and Yeoh, W. 2016.
\newblock Solving Goal Recognition Design Using {ASP}.
\newblock In \emph{AAAI}.

\bibitem[{von Tschammer, Mattm{\"u}ller, and Speck(2022)}]{von2022loopless}
von Tschammer, J.; Mattm{\"u}ller, R.; and Speck, D. 2022.
\newblock Loopless top-k planning.
\newblock In \emph{ICAPS}.

\bibitem[{Wayllace, Hou, and Yeoh(2017)}]{GRD_WayllaceH017}
Wayllace, C.; Hou, P.; and Yeoh, W. 2017.
\newblock New Metrics and Algorithms for Stochastic Goal Recognition Design Problems.
\newblock In \emph{IJCAI}.

\bibitem[{Wayllace et~al.(2016)Wayllace, Hou, Yeoh, and Son}]{GRD_WayllaceH0S16}
Wayllace, C.; Hou, P.; Yeoh, W.; and Son, T.~C. 2016.
\newblock Goal Recognition Design with Stochastic Agent Action Outcomes.
\newblock In \emph{IJCAI}.

\bibitem[{Wayllace et~al.(2020)Wayllace, Keren, Gal, Karpas, Yeoh, and Zilberstein}]{SGRD_WayllaceKGK0Z20}
Wayllace, C.; Keren, S.; Gal, A.; Karpas, E.; Yeoh, W.; and Zilberstein, S. 2020.
\newblock Accounting for Observer's Partial Observability in Stochastic Goal Recognition Design.
\newblock In \emph{ECAI}.

\bibitem[{Wayllace and Yeoh(2022)}]{GRD_Wayllace022}
Wayllace, C.; and Yeoh, W. 2022.
\newblock Stochastic Goal Recognition Design Problems with Suboptimal Agents.
\newblock In \emph{AAAI}.

\bibitem[{Zhang, Chen, and Parkes(2009)}]{ERD_IJCAI_ZhangCP09}
Zhang, H.; Chen, Y.; and Parkes, D.~C. 2009.
\newblock A General Approach to Environment Design with One Agent.
\newblock In \emph{IJCAI}.

\end{thebibliography}

\end{document}